\documentclass[10pt,twocolumn,letterpaper]{article}

\usepackage{iccv}
\usepackage{times}
\usepackage{epsfig}
\usepackage{graphicx}
\usepackage{amsmath}
\usepackage{amssymb}
\usepackage{bm}
\usepackage{amsthm}
\usepackage{multirow}
\usepackage{adjustbox}
\usepackage{caption}
\usepackage{subcaption}
\usepackage{mathtools}

\newtheorem{theorem}{Theorem}

\newcommand{\E}{\mathbb{E}}
\def \xx {\mathbf{x}}

\def \X  {\mathcal{X}}
\def \Y  {\mathcal{Y}}

\usepackage[pagebackref=true,breaklinks=true,letterpaper=true,colorlinks,bookmarks=false]{hyperref}

\iccvfinalcopy 

\def\iccvPaperID{2580} 

\ificcvfinal\fi
\begin{document}

\title{Symmetric Cross Entropy for Robust Learning with Noisy Labels}

\author{
Yisen Wang\textsuperscript{1}\textsuperscript{\footnotemark[1]} \footnotemark[2] \ \  Xingjun Ma\textsuperscript{2}\textsuperscript{\footnotemark[1]} \footnotemark[2] \ \ Zaiyi Chen\textsuperscript{3}\ \ Yuan Luo\textsuperscript{4}\ \ Jinfeng Yi\textsuperscript{1}\ \  James Bailey\textsuperscript{2}\\
  \textsuperscript{1}JD AI \ \ \ \ \ 
  \textsuperscript{2}The University of Melbourne \ \ \ \ \ \textsuperscript{3}Cainiao AI \ \ \ \ \ \textsuperscript{4}Shanghai Jiao Tong University  \\
}

\maketitle

\renewcommand{\thefootnote}{\fnsymbol{footnote}} 
\footnotetext[1]{Equal contribution.} 
\footnotetext[2]{Correspondence to: Yisen Wang (eewangyisen@gmail.com) and Xingjun Ma (xingjun.ma@unimelb.edu.au). } 

\thispagestyle{empty}

\begin{abstract}
Training accurate deep neural networks (DNNs) in the presence of noisy labels is an important and challenging task. Though a number of approaches have been proposed for learning with noisy labels, many open issues remain. In this paper, we show that DNN learning with Cross Entropy (CE) exhibits overfitting to noisy labels on some classes (``easy" classes), but more surprisingly, it also suffers from significant under learning on some other classes (``hard" classes). Intuitively, CE requires an extra term to facilitate learning of hard classes, and more importantly, this term should be noise tolerant, so as to avoid overfitting to noisy labels. Inspired by the symmetric KL-divergence, we propose the approach of \textbf{Symmetric cross entropy Learning} (SL), boosting CE symmetrically with a noise robust counterpart Reverse Cross Entropy (RCE). Our proposed SL approach simultaneously addresses both the under learning and overfitting problem of CE in the presence of noisy labels. We provide a theoretical analysis of SL and also empirically show, on a range of benchmark and real-world datasets, that SL outperforms state-of-the-art methods. We also show that SL can be easily incorporated into existing methods in order to further enhance their performance.
\end{abstract}

\section{Introduction}
Modern deep neural networks (DNNs) are often highly complex models that have hundreds of layers and millions of trainable parameters, requiring large-scale datasets with clean label annotations such as ImageNet \cite{deng2009imagenet} for proper training. However, labeling large-scale datasets is a costly and error-prone process, and even high-quality datasets are likely to contain noisy (incorrect) labels. Therefore, training accurate DNNs in the presence of noisy labels has become a task of great practical importance in deep learning.

Recently, several works have studied the dynamics of DNN learning with noisy labels. Zhang \textit{et.al} \cite{zhang2016understanding} argued that DNNs exhibit memorization effects whereby they first memorize the training data for clean labels and then subsequently memorize data for the noisy labels. Similar findings are also reported in \cite{arpit2017closer} that DNNs first learn clean and easy patterns and eventually memorize the wrongly assigned labels.  Further evidence is provided in \cite{ma2018dimensionality} that DNNs first learn simple representations via subspace dimensionality compression and then overfit to noisy labels via subspace dimensionality expansion. Different findings are reported in \cite{shwartz2017opening}, where DNNs with a specific activation function (\textit{i.e.}, tanh) undergo an initial label fitting phase then a subsequent representation compression phase where the overfitting starts. Despite these important findings, a complete understanding of DNN learning behavior, particularly their learning process for noisy labels, remains an open question.

\begin{figure}[!t]
	\centering
	\begin{subfigure}{0.49\linewidth}
		\includegraphics[width=\textwidth]{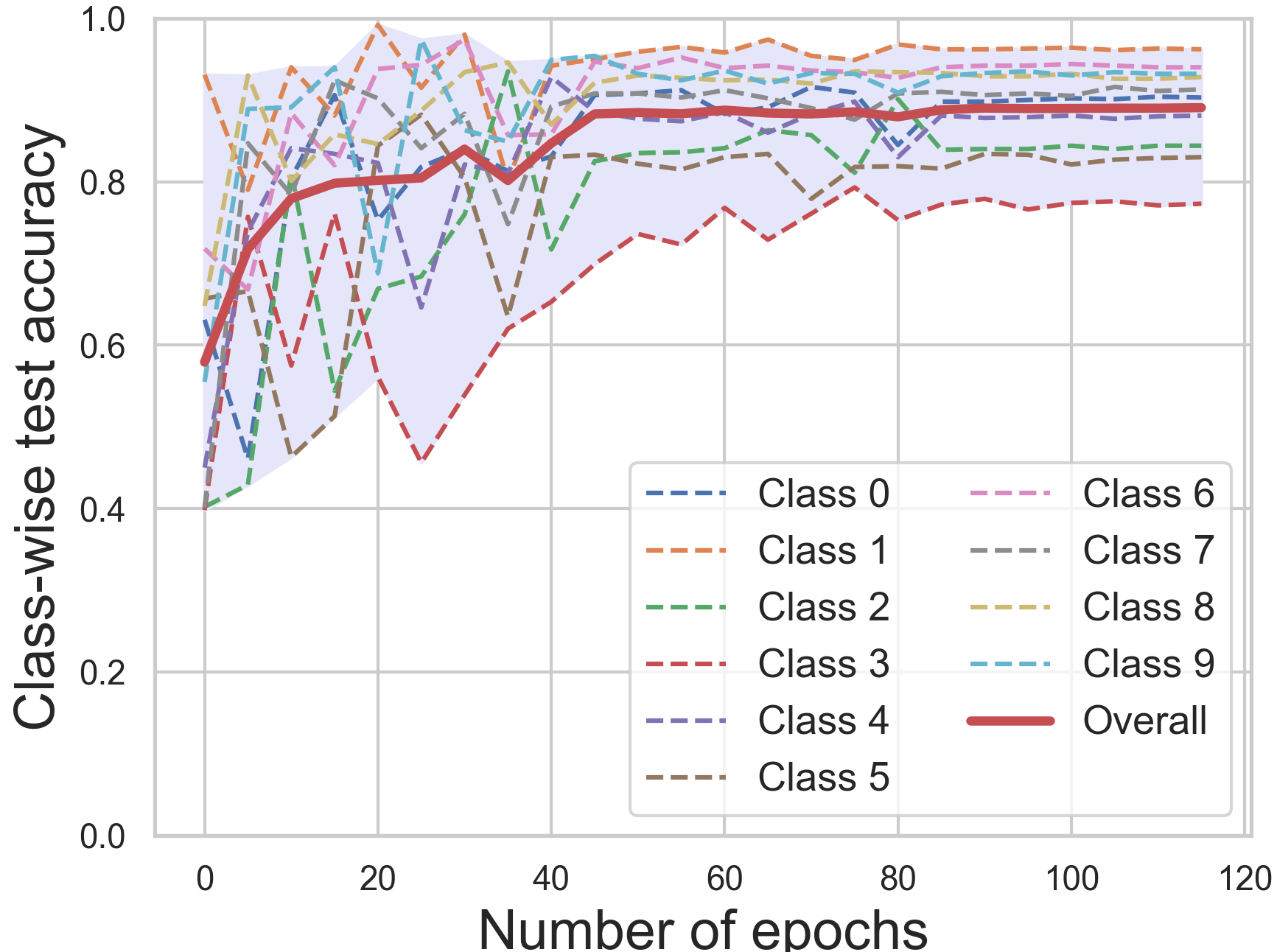}
		\caption{CE - clean}
		\label{ce_clean}
	\end{subfigure}
	\begin{subfigure}{0.49\linewidth} 
		\includegraphics[width=\textwidth]{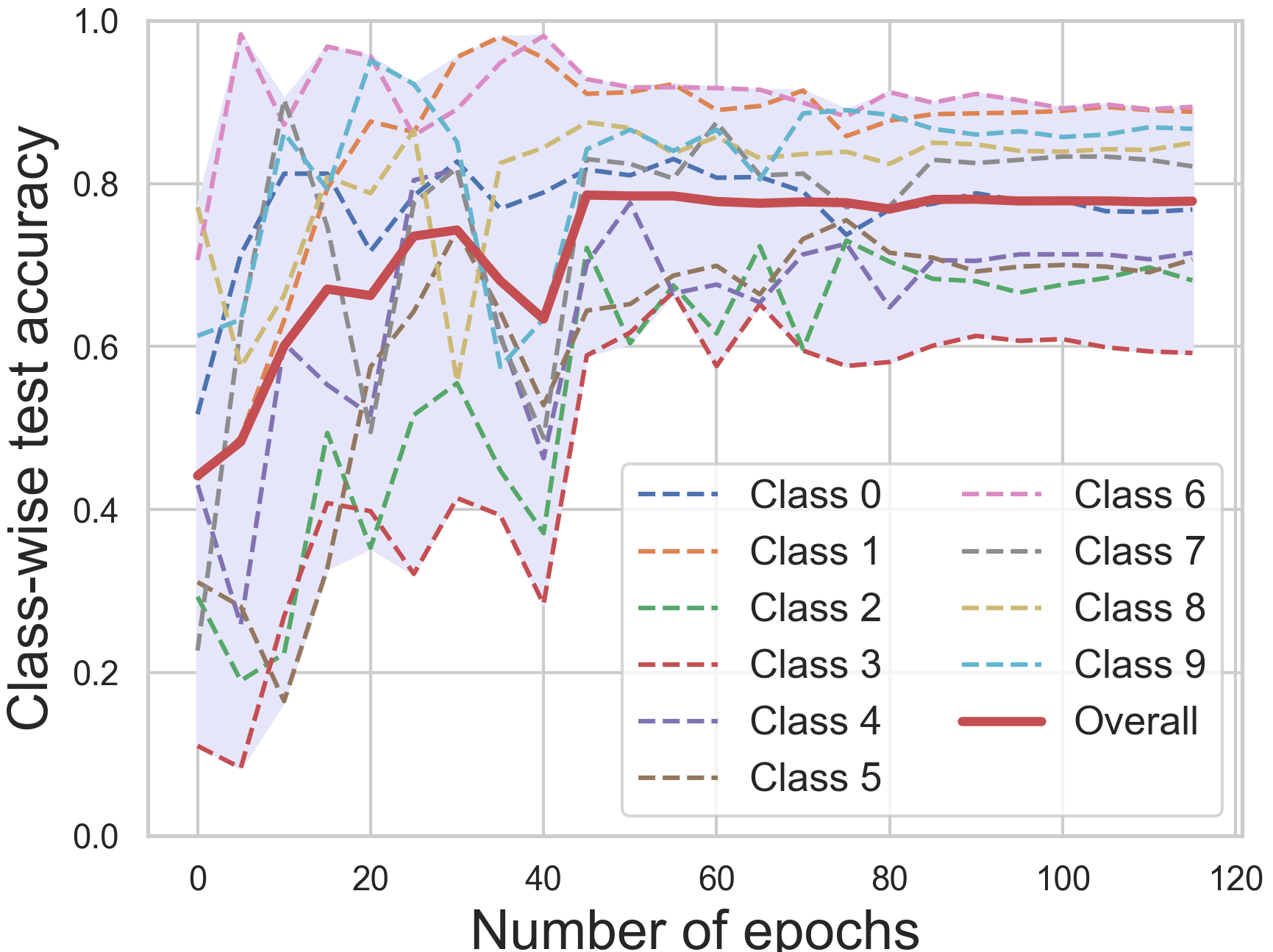}
		\caption{CE - noisy} 
		\label{ce_noisy}
	\end{subfigure}\\
	\begin{subfigure}{0.49\linewidth}
		\includegraphics[width=\textwidth]{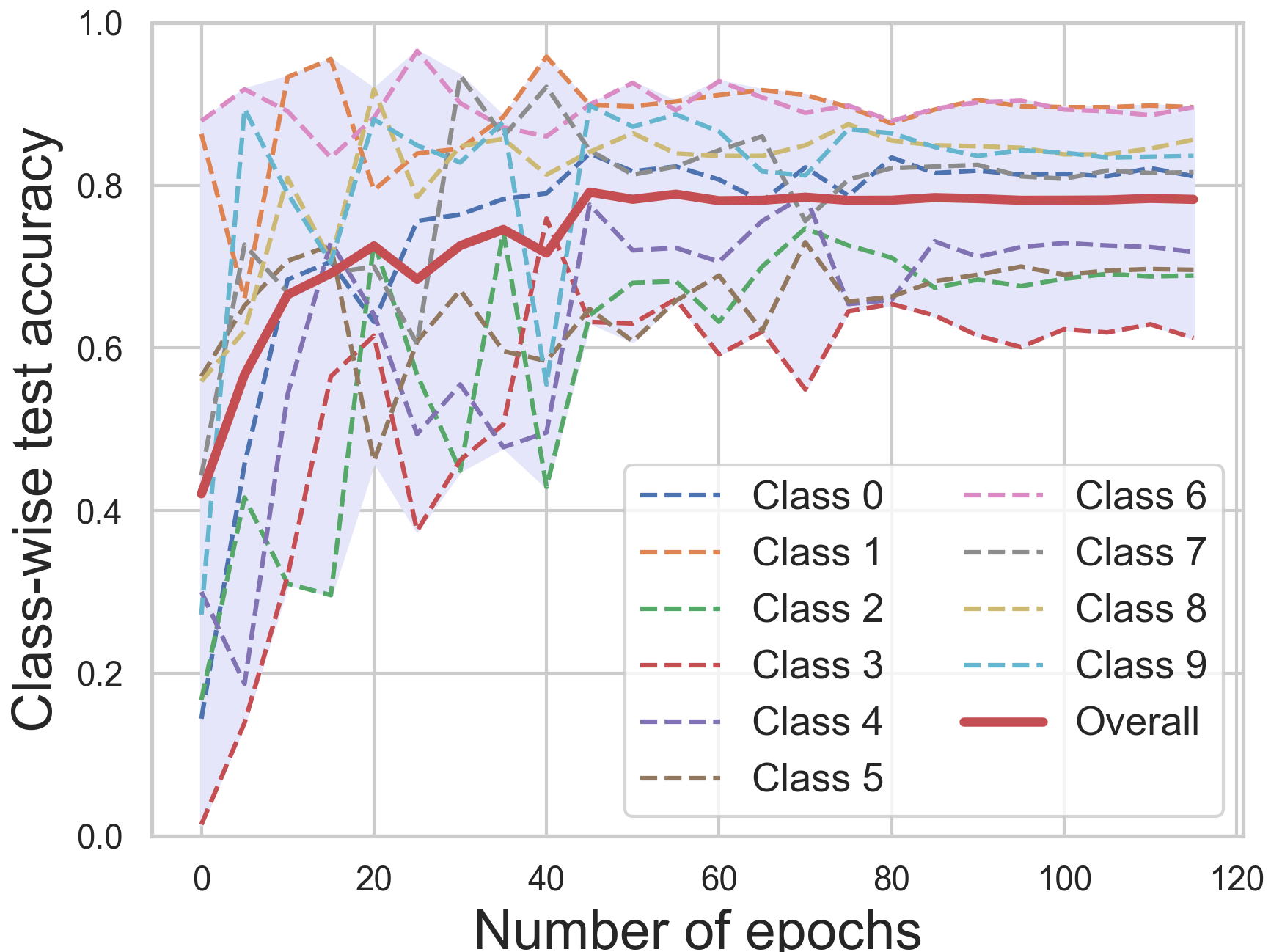}
		\caption{LSR - noisy}
		\label{lsr_noisy}
	\end{subfigure}
	\begin{subfigure}{0.49\linewidth}
		\includegraphics[width=\textwidth]{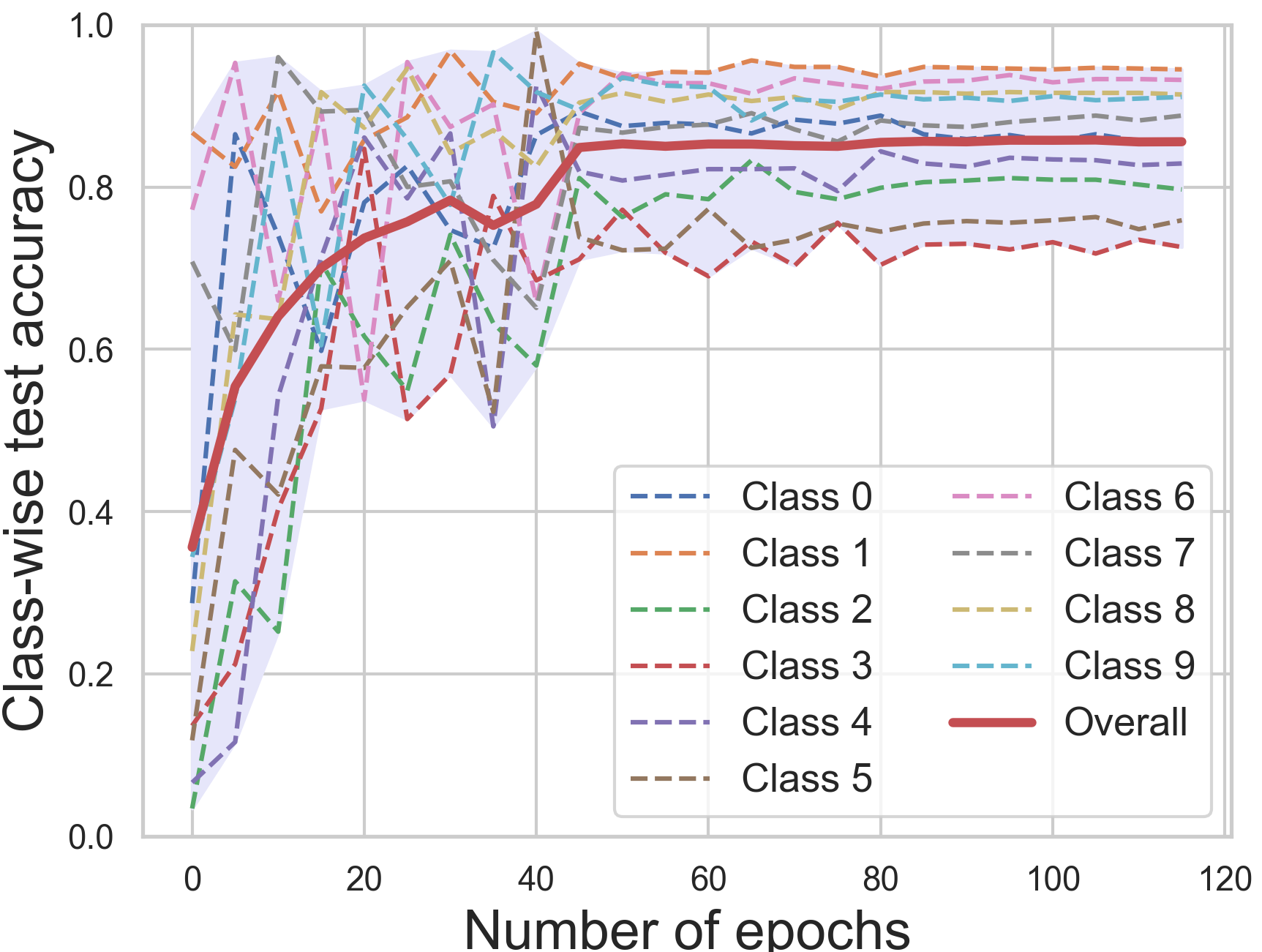}
		\caption{\textbf{SL} - noisy}
		\label{sce_noisy}
	\end{subfigure}
	\vspace{-0.1 in}
	\caption{The class-wise test accuracy of an 8-layer CNN on CIFAR-10 trained by (a) CE on clean labels with class-biased phenomenon, (b) CE on 40\% symmetric/uniform noisy labels with amplified class-biased phenomenon and under learning on hard classes (\textit{e.g.}, class 3), (c) LSR under the same setting to (b) with under learning on hard classes still existing, (d) our proposed SL under the same setting to (b) exhibiting overall improved learning on all classes.}
	\label{figure1}
	\vspace{-0.15 in}
\end{figure}

In this paper, we provide further insights into the learning procedure of DNNs by investigating the learning dynamics across classes. While Cross Entropy (CE) loss is the most commonly used loss for training DNNs, we have found that DNN learning with CE can be \textbf{\emph{class-biased}}: some classes (``easy" classes) are easy to learn and converge faster than other classes (``hard" classes). As shown in Figure~\ref{ce_clean}, even when labels are clean, the class-wise test accuracy spans a wide range during the entire training process. 
As further shown in Figure~\ref{ce_noisy}, this phenomenon is amplified when training labels are noisy: whilst easy classes (\textit{e.g.}, class 6) already overfit to noisy labels, hard classes (\textit{e.g.}, class 3) still suffer from significant \textbf{\emph{under learning}} (class accuracy significantly lower than clean label setting). Specifically, class 3 (bottom curve) only has an accuracy of ${\small\sim}60\%$ at the end, considerably less than the ${\small>}90\%$ accuracy of class 6 (top curve). Label Smoothing Regularization (LSR) \cite{szegedy2016rethinking,pereyra2017regularizing} is a widely known technique to ease overfitting issues, as shown in Figure \ref{lsr_noisy}, which still exhibits significant under learning on hard classes. Comparing the overall test accuracy (solid red curve) in Figure \ref{figure1}, a low test accuracy (under learning) on hard classes is a barrier to high overall accuracy. This is a different finding from previous belief that poor performance is simply caused by overfitting to noisy labels. We also visualize the learned representations for the noisy label case in Figure \ref{rep_ce_noisy}: some clusters are learned comparably well to those learned with clean labels (Figure \ref{rep_ce_clean}), while some other clusters do not have clear separated boundaries.

\begin{figure}[!t]
	\centering
	\begin{subfigure}{0.32\linewidth}
		\includegraphics[width=\textwidth]{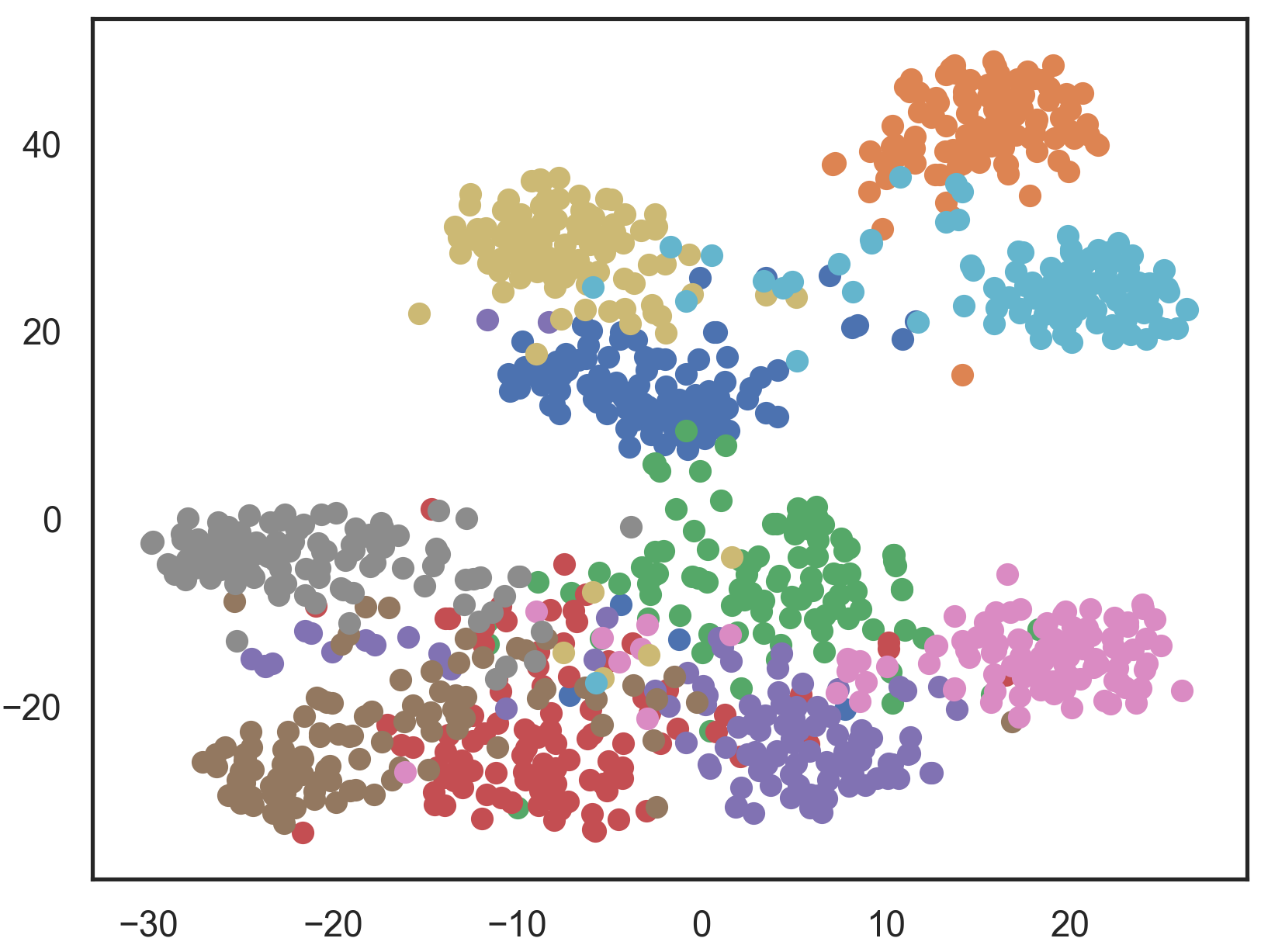}
		\caption{CE - clean}
		\label{rep_ce_clean}
	\end{subfigure}
	\begin{subfigure}{0.32\linewidth} 
		\includegraphics[width=\textwidth]{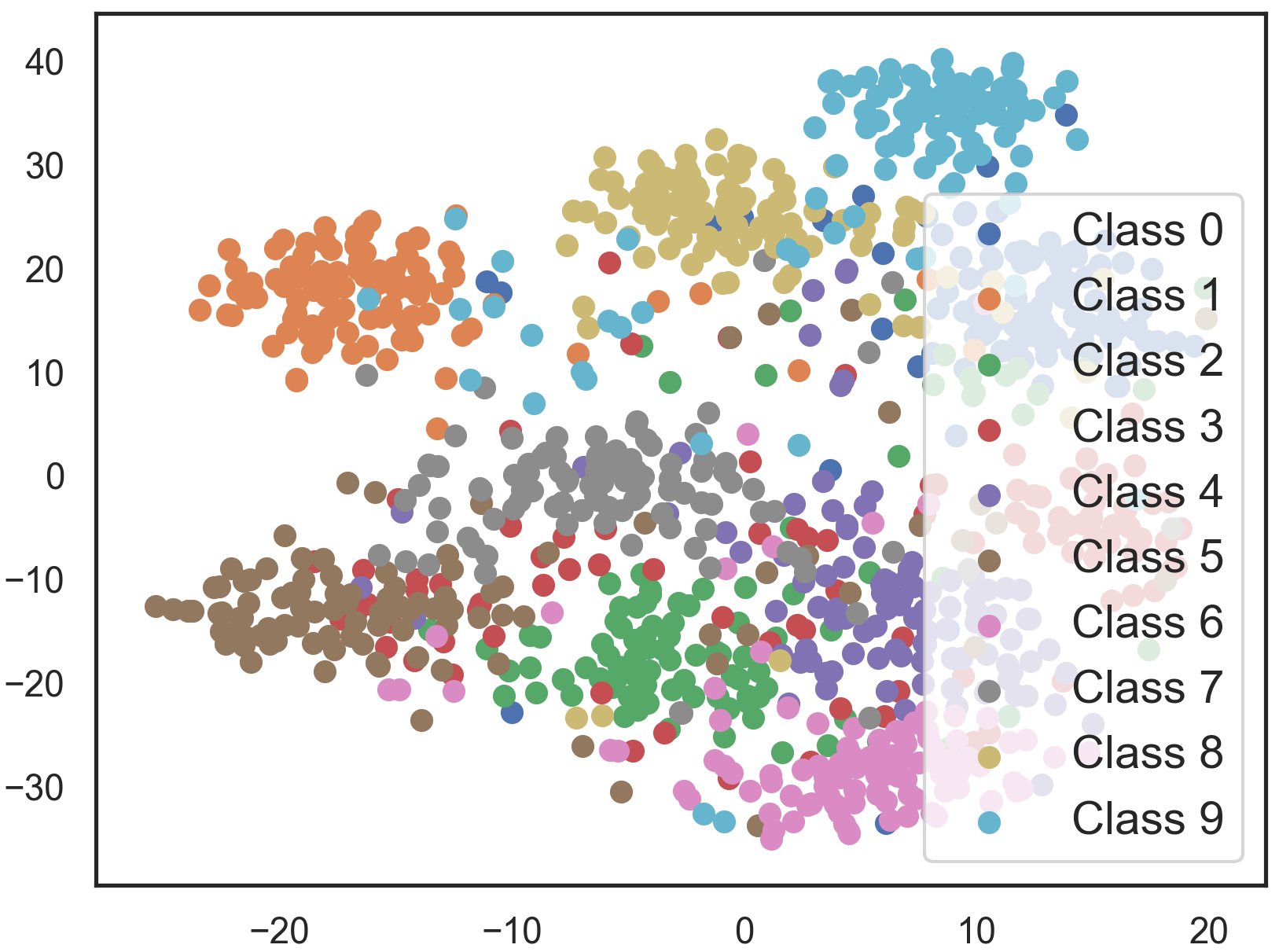}
		\caption{CE - noisy} 
		\label{rep_ce_noisy}
	\end{subfigure}
	\begin{subfigure}{0.32\linewidth}
		\includegraphics[width=\textwidth]{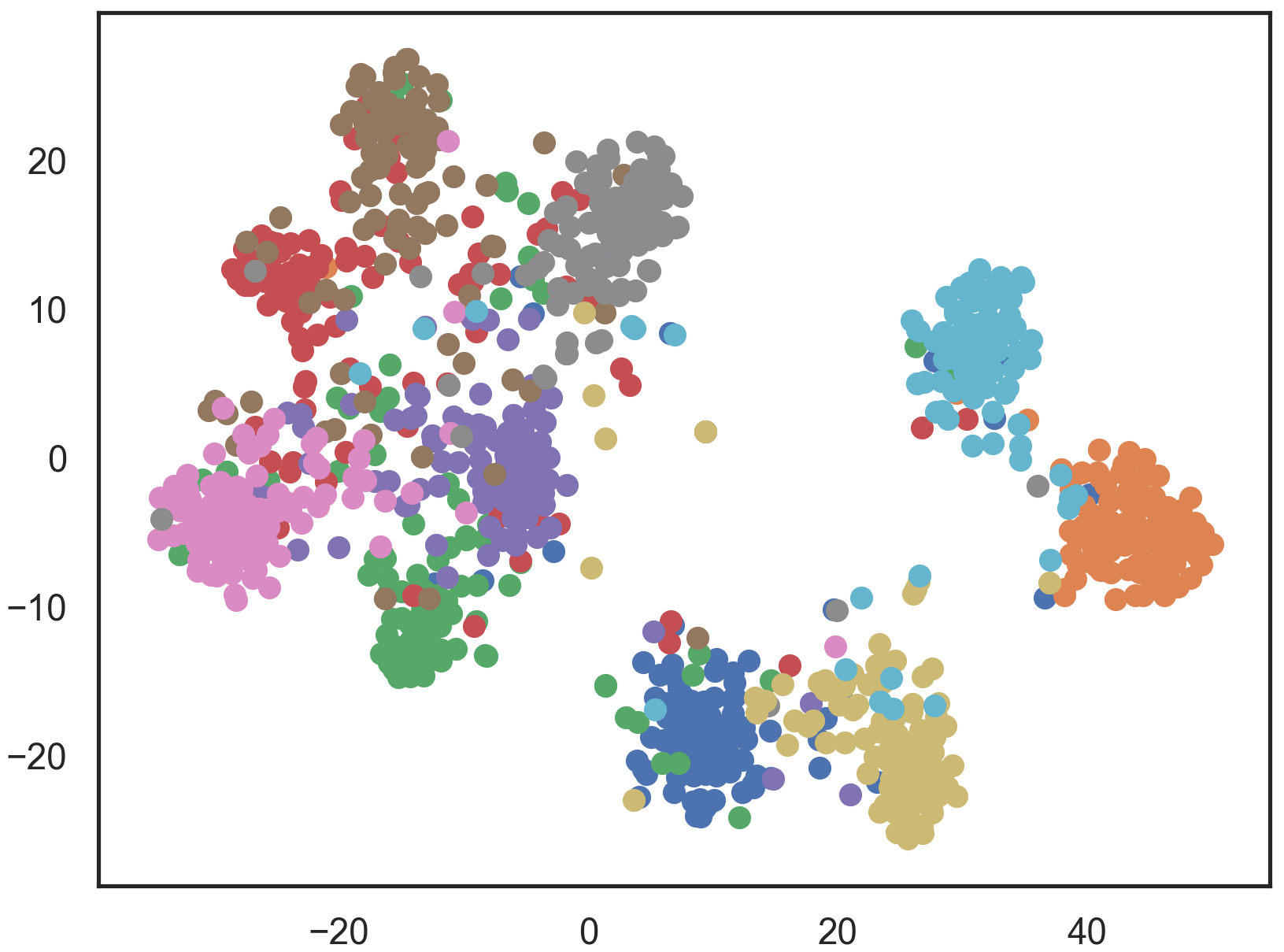}
		\caption{SL - noisy}
		\label{rep_sce_noisy}
	\end{subfigure}
	\vspace{-0.1 in}
	\caption{Visualization of learned representations on CIFAR-10 using t-SNE 2D embeddings of deep features at the last second dense layer with (a) CE on clean labels, (b) CE on 40\% symmetric noisy labels, (c) the proposed SL on the same setting to (b).}
	\label{ce_sce_rep_40}
	\vspace{-0.15 in}
\end{figure}

Intuitively, CE requires an extra term to improve its learning on hard classes, and more importantly, this term needs to be tolerant to label noise.
Inspired by the symmetric KL-divergence, we propose such a noise tolerant term, namely Reverse Cross Entropy (RCE), which combined with CE forms the basis of the approach Symmetric cross entropy Learning (SL). SL not only promotes sufficient learning (class accuracy close to clean label setting) of hard classes, but also improves the robustness of DNNs to noisy labels. As a preview of this, we can inspect the improved learning curves of class-wise test accuracy and representations in Figure~\ref{sce_noisy} and \ref{rep_sce_noisy}. Under the same 40\% noise setting, the variation of class-wise test accuracy has been narrowed by SL to 20\% with 95\% the highest and 75\% the lowest (Figure~\ref{sce_noisy}), and the learned representations are of better quality with more separated clusters (Figure~\ref{rep_sce_noisy}), both of which are very close to the clean settings. 

Compared to existing approaches that often involve architectural or non-trivial algorithmic modifications, SL is extremely simple to use. It requires minimal intervention to the training process and thus can be straightforwardly incorporated into existing models to further enhance their performance. In summary, our main contributions are:
\begin{itemize}
  \item  We provide insights into the class-biased learning procedure of DNNs with CE loss and find that the under learning problem of hard classes is a key bottleneck for learning with noisy labels.
  
  \item We propose a Symmetric Learning (SL) approach, to simultaneously address the hard class under learning problem and the noisy label overfitting problem of CE. We provide both theoretical analysis and empirical understanding of SL.
  
  \item We empirically demonstrate that SL can achieve better robustness than state-of-the-art methods, and can be also easily incorporated into existing methods to significantly improve their performance.
\end{itemize}

\section{Related Work}\label{sec:related_work}
Different approaches have been proposed to train accurate DNNs with noisy labels, and they can be roughly divided into three categories: 1) label correction methods, 2) loss correction methods, and 3) refined training strategies.

The idea of label correction is to improve the quality of the raw labels. One common approach is to correct noisy labels to their true labels via a clean label inference step using complex noise models characterized by directed graphical models \cite{xiao2015learning}, conditional random fields \cite{vahdat2017toward}, neural networks \cite{lee2017cleannet,veit2017learning} or knowledge graphs \cite{li2017learning}. These methods require the support from extra clean data or an expensive detection process to estimate the noise model.

Loss correction methods modify the loss function for robustness to noisy labels. One approach is to model the noise transition matrix that defines the probability of one class changed to another class \cite{han2018masking}. Backward \cite{patrini2017making} and Forward \cite{patrini2017making} are two such correction methods that use the noise transition matrix to modify the loss function. However, the ground-truth noise transition matrix is not always available in practice, and it is also difficult to obtain accurate estimation \cite{han2018masking}. Work in \cite{goldberger2016training,sukhbaatar2014training} augments the correction architecture by adding a linear layer on top of the neural network. Bootstrap \cite{reed2014training}  uses a combination of raw labels and their predicted labels. There is also research that defines noise robust loss functions, such as Mean Absolute Error (MAE) \cite{ghosh2017robust}, but a challenge is that training a network with MAE is slow due to gradient saturation.    Generalized Cross Entropy (GCE) loss \cite{zhang2018generalized} applies a Box-Cox transformation to probabilities (power law function of probability with exponent $q$) and can behave like a weighted MAE. Label Smoothing Regularization (LSR) \cite{szegedy2016rethinking,pereyra2017regularizing} is another technique using soft labels in place of one-hot labels to alleviate overfitting to noisy labels.

Refined training strategies design new learning paradigms for noisy labels. MentorNet \cite{jiang2018mentornet,yu2019does} supervises the training of a StudentNet by a learned sample weighting scheme in favor of probably correct labels. Decoupling training strategy \cite{malach2017decoupling} trains two networks simultaneously, and parameters are updated when their predictions disagree. Co-teaching \cite{han2018co} maintains two networks simultaneously during training, with one network learning from the other network's most confident samples. These studies all require training of an auxiliary network for sample weighting or learning supervision. D2L \cite{ma2018dimensionality} uses subspace dimensionality adapted labels for learning, paired with a training process monitor. The iterative learning framework~\cite{wang2018iterative} iteratively detects and isolates noisy samples during the learning process. The joint optimization framework~\cite{tanaka2018joint} updates DNN parameters and labels alternately. These methods either rely on complex interventions into the learning process, which may be challenging to adapt and tune, or are sensitive to hyperparameters like the number of training epochs and learning rate.

\section{Weakness of Cross Entropy}\label{sec:understanding}
We begin by analyzing the Cross Entropy (CE) and its limitations for learning with noisy labels.

\subsection{Preliminaries}
Given a $K$-class dataset $\mathcal{D} = \{(\xx, y)^{(i)}\}_{i=1}^n$, with $\xx \in \X \subset \mathbb{R}^d$ denoting a sample in the $d$-dimensional input space and $y \in \Y = \{1, \cdots, K\}$ its associated label. For each sample $\xx$, a classifier $f(\xx)$ computes its probability of each label $k \in \{1, \cdots, K\}$: $p(k|\xx) = \frac{e^{z_{k}}}{\sum_{j=1}^K e^{z_{j}}}$, where $z_j$ are the logits. We denote the ground-truth distribution over labels for sample $\xx$ by $q(k|\xx)$, and $\sum_{k=1}^{K}q(k|\xx)=1$. Consider the case of a single ground-truth label $y$, then $q(y|\xx)=1$ and $q(k|\xx)=0$ for all $k \neq y$. The cross entropy loss for sample $\xx$ is:
\vspace{-0.1 in}
\begin{equation}\label{eq:ce}
\begin{split}
    \ell_{ce} = -\sum_{k=1}^{K} q(k|\xx) \log p(k|\xx).
\end{split}
\end{equation}

\subsection{Weakness of CE under Noisy Labels}\label{drawback_ce}
We now highlight some weaknesses of CE for DNN learning with noisy labels, based on empirical evidence on CIFAR-10 dataset \cite{krizhevsky2009learning} (10 classes of natural images). To generate noisy labels, we randomly flip a correct label to one of the other 9 incorrect labels uniformly (\textit{e.g.}, symmetric label noise), and refer to the portion of incorrect labels as the noise rate. The network used here is an 8-layer convolutional neural network (CNN). Detailed experimental settings can be found in Section~\ref{understanding_sce}.

We first explore in more detail the class-biased phenomenon shown in Figure \ref{ce_clean} and \ref{ce_noisy}, focusing on three distinct learning stages: early (the $10$-th epoch), middle (the $50$-th epoch) and later (the $100$-th epoch) stages, with respect to total 120 epochs of training. 
As illustrated in Figure \ref{fig:biased_class_test_acc}, CE learning starts in a highly class-biased manner (the blue curves) for both clean labels and 40\% noisy labels. 
This is because patterns inside of samples are intrinsically different. For clean labels, the network eventually manages to learn all classes uniformly well, reflected by the relatively flat accuracy curve across classes (the green curve in Figure \ref{acc_ce_clean}). However, for noisy labels, the class-wise test accuracy varies significantly across different classes, even at the later stage (the green curve in Figure \ref{acc_ce_noisy}). In particular, the network struggles to learn hard classes (\textit{e.g.}, class $2/3$) with up to a 20\% gap to the clean setting, whereas some easy classes (\textit{e.g.}, class $1/6$) are better learned and already start overfitting to noisy labels (accuracy drops from epoch 50 to 100). It appears that the under learning of hard classes is a major cause for the overall performance degradation, due to the fact that the accuracy drop caused by overfitting is relatively small.

\begin{figure}[!t]
	\centering
	\begin{subfigure}{0.49\linewidth}
		\includegraphics[width=\textwidth]{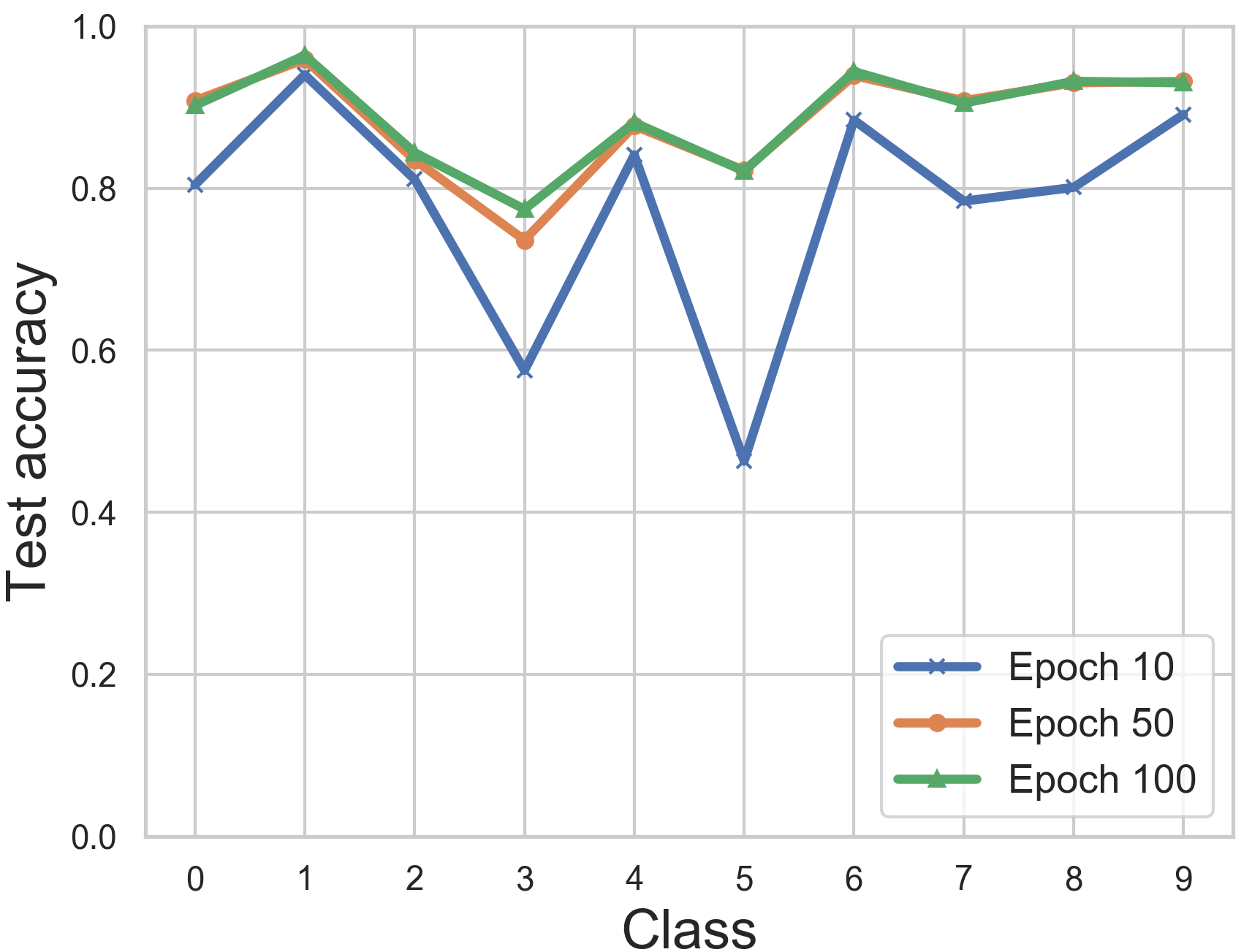}
		\caption{CE - clean}
		\label{acc_ce_clean}
	\end{subfigure}
	\begin{subfigure}{0.49\linewidth} 
		\includegraphics[width=\textwidth]{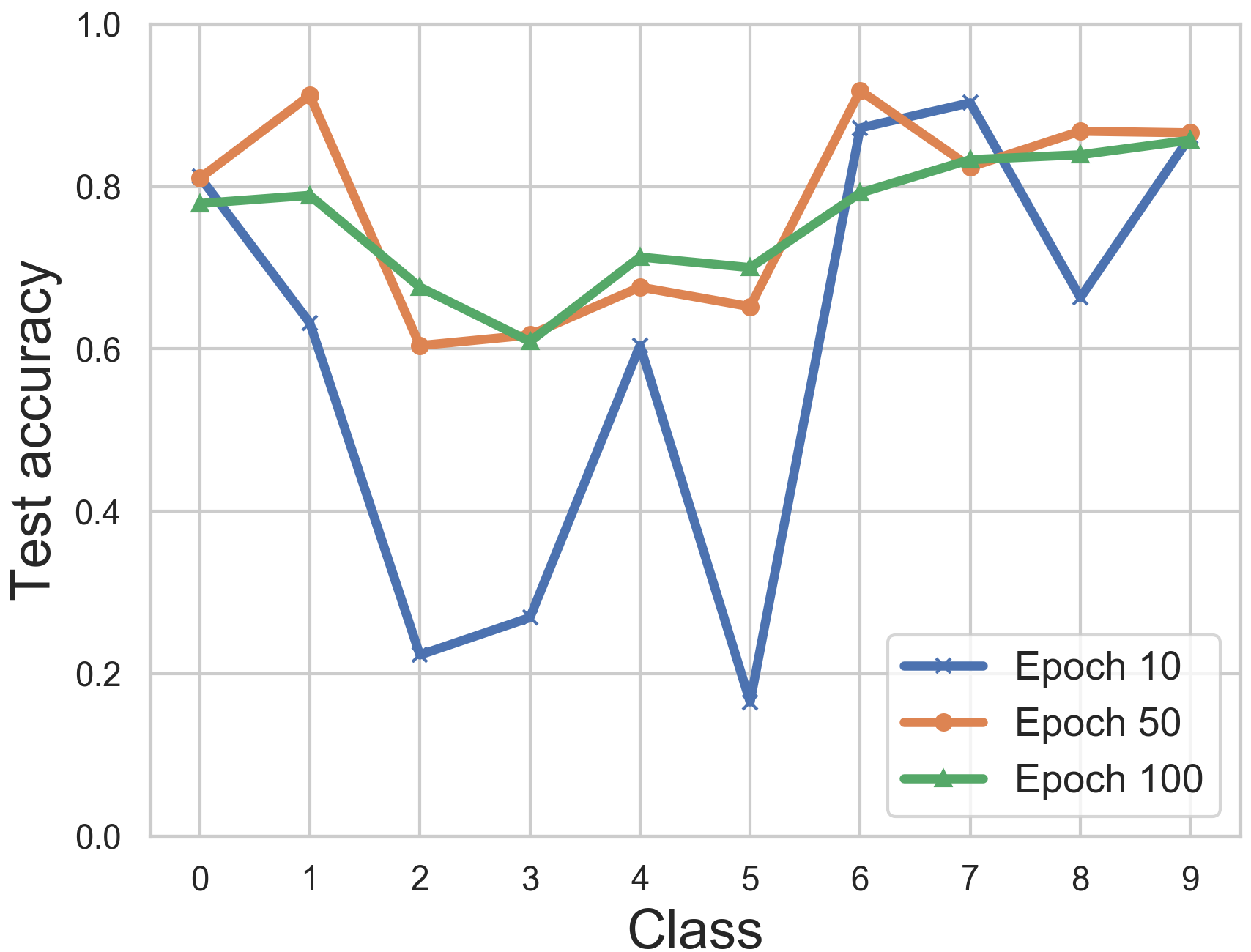}
		\caption{CE - noisy} 
		\label{acc_ce_noisy}
	\end{subfigure}
	\vspace{-0.1 in}
	\caption{The class-wise test accuracy at epoch 10, 50 and 100 (120 epochs in total) trained by CE loss on CIFAR-10 with (a) clean labels or (b) 40\% symmetric noisy labels.}
	\label{fig:biased_class_test_acc}
	\vspace{-0.1 in}
\end{figure}

\begin{figure}[!t]
	\centering
	\begin{subfigure}{0.49\linewidth}
		\includegraphics[width=\textwidth]{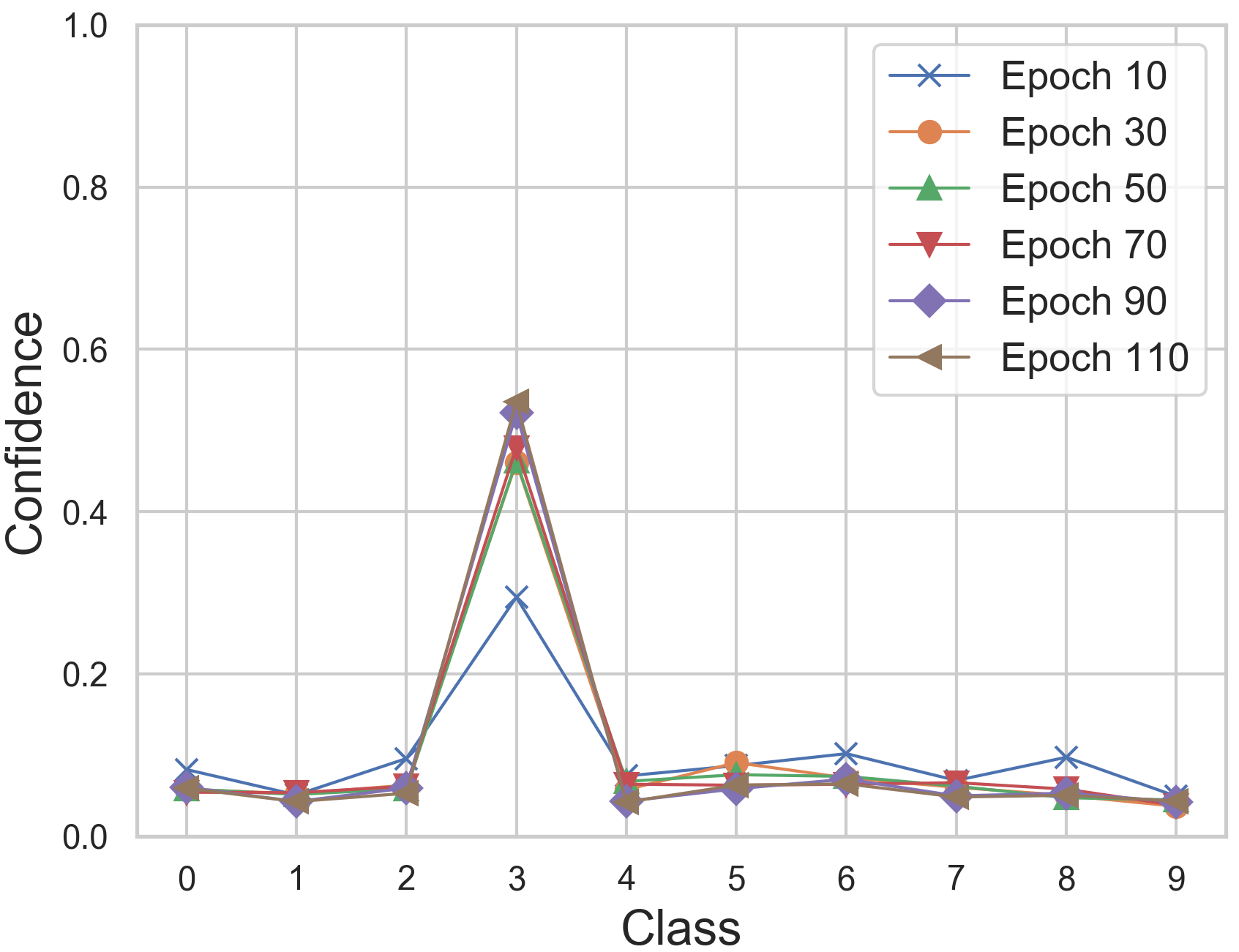}
		\caption{Prediction confidence}
		\label{ce_confidence}
	\end{subfigure}
	\begin{subfigure}{0.49\linewidth} 
		\includegraphics[width=\textwidth]{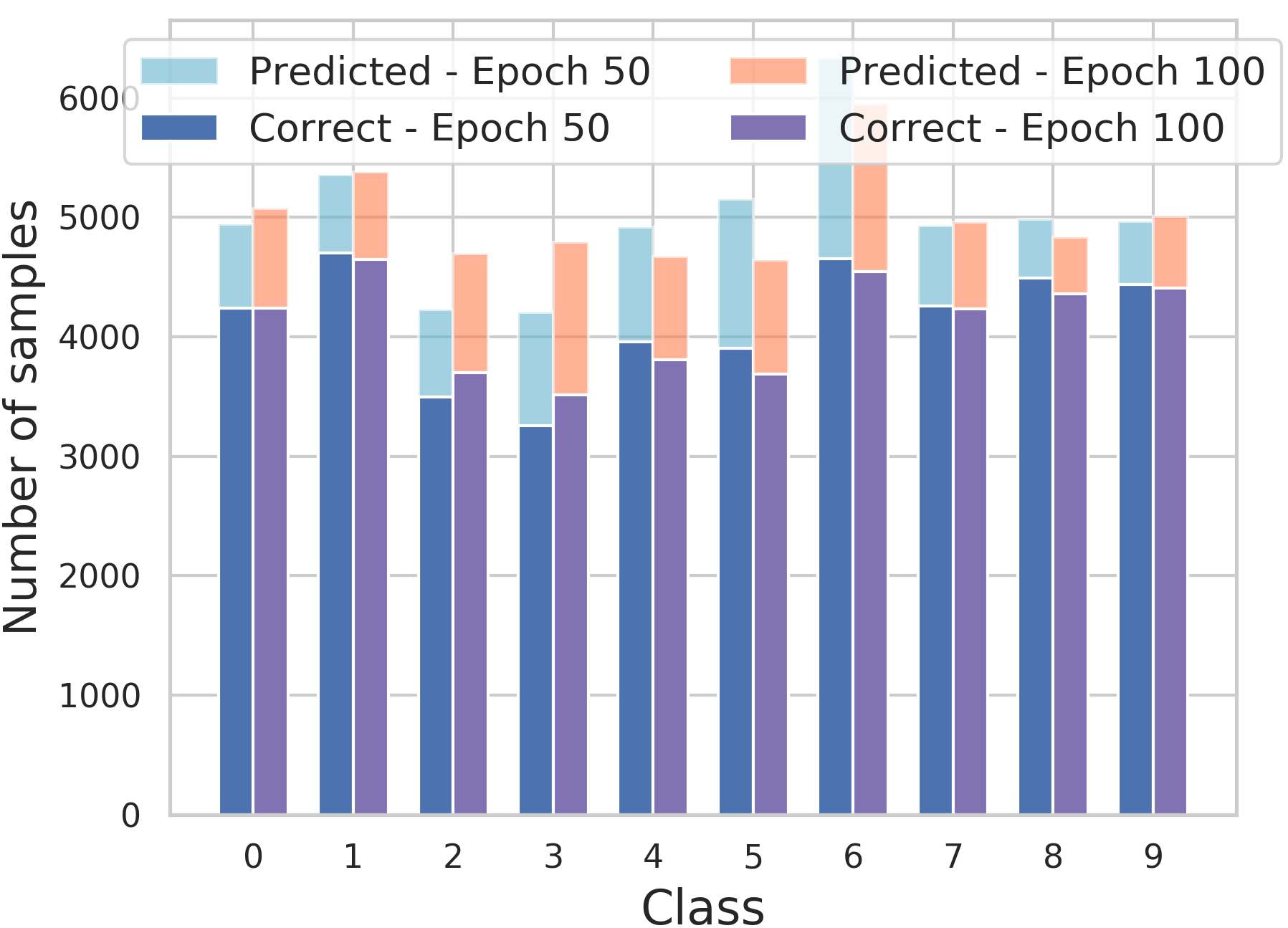}
		\caption{Prediction distribution} 
		\label{pred_imbalance}
	\end{subfigure}
	\vspace{-0.1 in}
	\caption{Intermediate results of CE loss on CIFAR-10 with 40\% symmetric noisy labels. (a) Average confidence of the \textbf{clean portion} of class 3 samples. (b) The true positive samples (\emph{correct}) out of predictions (\emph{predicted}) for each class.}
	\label{inner_ce}
	\vspace{-0.15 in}
\end{figure}

We further investigate the reason behind the under learning of CE on hard classes from the perspective of representations. Due to their high similarities in representations to some other classes (see the red cluster for class 3 in Figure \ref{rep_ce_clean}), the predictions for hard class examples are likely to assign a relatively large probability to those similar classes. 
Under the noisy label scenario, class 3 has become even more scattered into other classes (red cluster in Figure \ref{rep_ce_noisy}). As a consequence, no visible cluster was learned by CE, even though there are still 60\% correct labels in this scenario. Further delving into these 60\% \textbf{clean portion} of class 3 samples, we show, in Figure \ref{ce_confidence}, their prediction confidence output of the neural network. Although the confidence at class 3 is the highest, it is only around 0.5, while for the other classes, the confidence is around 0.05 or 0.1 which is actually a relatively high value and an indication of insufficient learning of class 3 even on the clean labeled part.
Another evidence of under learning can be obtained from Figure \ref{pred_imbalance}, where hard classes (\textit{e.g.}, class $2/3$) have fewer true positive samples throughout intermediate stages of learning.

Clearly, CE by itself is not sufficient for learning of hard classes, especially under the noisy label scenario. We note that this finding sheds new insights into DNN learning behavior under label noise, and differs from previous belief that DNNs overfit to all classes in general \cite{arpit2017closer,zhang2016understanding}. In the next section, we propose a symmetric learning approach that can address both the hard class under learning and noisy label overfitting problems of CE.

\section{Symmetric Cross Entropy Learning}\label{sec:sce}
In this section, we propose Symmetric cross entropy Learning (SL), an approach that strikes a balance between sufficient learning and robustness to noisy labels.  We also provide theoretical analysis about the formulation and behavior of SL.

\subsection{Definition}
Given two distributions $q$ and $p$, the relation between the cross entropy (denoted as $H(q, p)$) and the KL-divergence (denoted as $KL(q\| p)$) is:
\begin{equation}
\label{eq:ce_derivative}
KL(q \| p) = H(q, p) - H(q),
\end{equation}
where $H(q)$ is the entropy of $q$. In the context of classification, $q=q(k|\xx)$ is the ground truth class distribution conditioned on sample $\xx$, whilst $p=p(k|\xx)$ is the predicted distribution over labels by the classifier $f$. From the perspective of KL-divergence, classification is to learn a prediction distribution $p(k|\xx)$ that is close to the ground truth distribution $q(k|\xx)$, which is to minimize the KL-divergence $KL(q\| p)$ between the two distributions\footnote{In practice, the $H(q(k|\xx))$ term is a constant for a given class distribution and therefore omitted from Eq.~\eqref{eq:ce_derivative} giving the CE loss in Eq.~\eqref{eq:ce}.}.

In information theory, given a true distribution $q$ and its approximation $p$, $KL(q\|p)$ measures the penalty on encoding samples from $q$ using code optimized for $p$ (penalty in the number of extra bits required). In the context of noisy labels, we know that $q(k|\xx)$ does not represent the true class distribution, instead $p(k|\xx)$ can reflect the true distribution to a certain extent. Thus, in addition to taking $q(k|\xx)$ as the ground truth, we also need to consider the other direction of KL-divergence, that is $KL(p||q)$, to punish coding samples that come from $p(k|\xx)$ when using a code for $q(k|\xx)$. The symmetric KL-divergence is:
\begin{equation}\label{kl_sym}
SKL = KL(q||p) + KL(p||q).
\end{equation}

Transferring this symmetric idea from KL-divergence to cross entropy gives us the \textit{Symmetric Cross Entropy} (SCE):
\begin{equation}\label{ce_sym}
SCE = CE + RCE = H(q, p) + H(p, q),
\end{equation}
where $RCE=H(p, q)$ is the reverse version of $H(q, p)$, namely, \emph{Reverse Cross Entropy}. The RCE loss for a sample $\xx$ is:
\begin{equation}\label{eq:rce}
    \ell_{rce} = -\sum_{k=1}^{K} p(k|\xx) \log q(k|\xx).
\end{equation}

The sample-wise SCE loss can then be defined as:
\begin{equation}\label{eq:sce}
  \ell_{sce} = \ell_{ce} + \ell_{rce}.
\end{equation}

While the RCE term is noise tolerant as will be proved in Section \ref{sec:theor}, the CE term is not robust to label noise \cite{ghosh2017robust}. However, CE is useful for achieving good convergence \cite{zhang2018generalized}, which will be verified empirically in Section \ref{sec:experiments}. 
Towards more effective and robust learning, we propose a flexible symmetric learning framework with the use of two decoupled hyperparameters (\textit{e.g.}, $\alpha$ and $\beta$), with $\alpha$ on the overfitting issue of CE while $\beta$ for flexible exploration on the robustness of RCE. Formally, the SL loss is:
\begin{equation}\label{eq:sce_alpha}
   \ell_{sl} = \alpha \ell_{ce} + \beta \ell_{rce}.
\end{equation}

As the ground truth distribution $q(k|\xx)$ is now inside of the logarithm in $\ell_{rce}$, this could cause computational problem when labels are one-hot: zero values inside the logarithm. To solve this issue, we define $\log 0 = A$ (where $A<0$ is some constant), which shortly will be proved useful for the robustness of $\ell_{rce}$ in Theorem \ref{theorem_1}.
This technique is similar to the clipping operation implemented by most deep learning frameworks. Compared with another option label smoothing technique, our approach introduces less bias into the model (negligible bias (in the view of training) at finite number of points like $q(k|x)=0$ but no bias at $q(k|x)=1$). Note that, the effect of $\beta$ on RCE can be reflected by different settings of $A$ (refer to Eq. \eqref{MAE-RCE}).

\subsection{Theoretical Analysis}\label{sec:theor}
\noindent\textbf{Robustness analysis:}
In the following, we will prove that the RCE loss $\ell_{rce}$ is robust to label noise following \cite{ghosh2017robust}. We denote the noisy label of $\xx$ as $\hat{y}$, in contrast to its true label $y$. Given any classifier $f$ and loss function $\ell_{rce}$, we define the $risk$ of $f$ under clean labels as $R(f)=\E_{\xx,y}\ell_{rce}$, and the $risk$ under label noise rate $\eta$ as $R^{\eta}(f)=\E_{\xx, \hat{y}} \ell_{rce}$. Let $f^*$ and $f_{\eta}^*$ be the global minimizers of $R(f)$ and $R^{\eta}(f)$ respectively. Risk minimization under a given loss function is noise robust if $f_{\eta}^*$ has the same probability of misclassification as that of $f^*$ on noise free data. When the above is satisfied we also say that the loss function is noise-tolerant.

\begin{theorem}\label{theorem_1}
In a multi-class classification problem, $\ell_{rce}$ is noise tolerant under symmetric or uniform label noise if noise rate $\eta < 1 -\frac{1}{K}$. And, if $R(f^*) = 0$, $\ell_{rce}$ is also noise tolerant under asymmetric or class-dependent label noise when noise rate $\eta_{y k} < 1-\eta_{y}$ with $\sum_{k \neq y}\eta_{y k}=\eta_{y}$.
\end{theorem}
\begin{proof}
For symmetric noise:
\begin{align*}
	\small
	R^\eta(f) & =  \E_{\xx, \hat{y}} \ell_{rce} =  \E_{\xx} \E_{y | \xx} \E_{\hat{y} | \xx, y} \ell_{rce} \\
		&= \E_{\xx} \E_{y | \xx} \Big[ (1-\eta) \ell_{rce} + \frac{\eta}{K-1} \sum_{k\neq y}^{K} \ell_{rce} \Big] \\
		& =  (1 - \eta) R(f) +  \frac{\eta}{K-1} (\sum_{k=1}^{K}\ell_{rce} - R(f))\\
		& = R(f)\left(1-\frac{\eta K}{K-1}\right) - A \eta,
\end{align*}
where the last equality holds due to $\sum_{k=1}^{K}\ell_{rce} = - (K-1)A$ following Eq. \eqref{eq:rce} and definition of $\log 0 = A$. Thus, 
	\[R^\eta(f^*)-R^\eta(f)=(1-\frac{\eta K}{K-1})(R(f^*)-R(f)) \leq 0\]
because $\eta < 1 - \frac{1}{K}$ and $f^*$ is a global minimizer of $R(f)$. This proves $f^*$ is also the global minimizer of risk $R^\eta(f)$, that is, $\ell_{rce}$ is noise tolerate. 
	
Similarly, we can prove the case for asymmetric noise, please refer Appendix \ref{appendix_proof} for details.
\end{proof}

\noindent\textbf{Gradient analysis:}
We next derive the gradients of a simplified SL with $\alpha, \beta=1$ to give a rough idea of how its learning process differs from that of CE\footnote{Complete derivations can be found in the Appendix \ref{appendix_gradient}.}. For brevity, we denote $p_k$, $q_k$ as abbreviations for $p(k|\xx)$ and $q(k|\xx)$. Consider the case of a single true label, the gradient of the sample-wise RCE loss with respect to the logits $z_j$ can be derived as:
\begin{equation}
        \frac{\partial \ell_{rce}}{\partial z_j}  =
        - \sum_{k=1}^K \frac{\partial p_k}{\partial z_j} \log q_k,
\end{equation}
where $\frac{\partial p_k}{\partial z_j}$ can be further derived based on whether $k=j$:
\begin{equation}\label{eq:partial}
    \frac{\partial p_k}{\partial z_j}  = 
    \begin{cases}
    p_k(1-p_k), & k = j \\
    -p_j p_k, & k \neq j.
    \end{cases}
\end{equation}
According to Eq. \eqref{eq:partial} and the ground-truth distribution for the case of one single label (\textit{e.g.}, $q_y = 1$, and $q_k = 0$ for $k \neq y$), the gradients of SL can be derived as:
\begin{equation}
    \frac{\partial \ell_{sl}}{\partial z_j} = 
    \begin{cases}
    \frac{\partial \ell_{ce}}{\partial z_j} - (Ap^{2}_{j} - Ap_j), & q_j = q_y = 1  \\
    \frac{\partial \ell_{ce}}{\partial z_j} + (- Ap_{j}p_{y}), & q_j = 0,
    \end{cases}
\end{equation}
where $A$ is the smoothed/clipped replacement of $\log 0$. Note that the gradient of sample-wise CE loss $\ell_{ce}$ is:
\begin{equation}
    \frac{\partial \ell_{ce}}{\partial z_j} = 
    \begin{cases}
    p_j - 1 \leq 0, & q_j = q_y = 1  \\
    p_j \geq 0, & q_j = 0.
    \end{cases}
\end{equation}

In the case of $q_j = q_y = 1$ ($\frac{\partial \ell_{ce}}{\partial z_j} \leq 0$), the second term $Ap^{2}_{j} - Ap_j$ is an adaptive acceleration term based on $p_j$. Specifically, $Ap^{2}_{j} - Ap_j$ is a convex parabolic function in the first quadrant for $p_j \in [0, 1]$, and has the maximum value at $p_j=0.5$. Considering the learning progresses towards $p_j \to 1$, RCE increases DNN prediction on label $y$ with larger acceleration for $p_j < 0.5$ and smaller acceleration for $p_j > 0.5$. In the case of $q_j = 0$ ($\frac{\partial \ell_{ce}}{\partial z_j} \geq 0$), the second term $-Ap_{j}p_{y}$ is an adaptive acceleration on the minimization of the probability at unlabeled class ($p_j$), based on the confidence at the labeled class ($p_y$). Larger $p_y$ leads to larger acceleration, that is, if the network is more confident about its prediction at the labeled class, then the residual probabilities at other unlabeled classes should be reduced faster. When $p_y=0$, there is no acceleration, which means if the network is not confident on the labeled class at all, then the label is probably wrong, no acceleration needed.

\subsection{Discussion}\label{sec:discuss}
An easy addition to improve CE would be to upscale its gradients with a larger coefficient (\textit{e.g.}, `2CE', `5CE'). However, this will cause more overfitting (see the `5CE' curve in the following Section \ref{sec:experiments} Figure \ref{ablation_sce}). There are also other options to consider, such as MAE. Although motivated from completely different perspectives, that is, CE and RCE are measures of (information theoretic) uncertainty, while MAE is a measure of distance, we can surprisingly show that MAE is a special case of RCE at $A=-2$, when there is a single true label for $\xx$ (\textit{e.g.} $q(y|\xx)=1$ and $q(k\neq y|\xx)=0$). For MAE, we have,
\begin{equation*}
    \begin{split}
        \ell_{mae} &= \sum_{k=1}^{K} |p(k|\xx) - q(k|\xx)|
         = (1-p(y|\xx)) + \sum_{k \neq y} p(k|\xx) \\
        & = 2(1-p(y|\xx)),
    \end{split}
\end{equation*}
while, for RCE, we have,
\begin{equation*}
    \begin{split}
        \ell_{rce} &= -\sum_{k=1}^{K} p(k|\xx) \log q(k|\xx) \\
        & = -p(y|\xx)\log1 - \sum_{k \neq y} p(k|\xx)A 
        = -A\sum_{k \neq y} p(k|\xx) \\
        & = -A(1-p(y|\xx)).
    \end{split}
    \label{MAE-RCE}
\end{equation*}
That is, when $A=-2$, RCE is reduced to exactly MAE. Meanwhile, different from the GCE loss (\textit{i.e.}, a weighted MAE) \cite{zhang2018generalized},  SL is a combination of two symmetrical learning terms. 

\section{Experiments}\label{sec:experiments}
We first provide some empirical understanding of our proposed SL approach, then evaluate its robustness against noisy labels on MNIST, CIFAR-10, CIFAR-100, and a large-scale real-world noisy dataset Clothing1M.

\noindent\textbf{Noise setting:} 
We test two types of label noise: symmetric (uniform) noise and asymmetric (class-dependent) noise. Symmetric noisy labels are generated by flipping the labels of a given proportion of training samples to one of the other class labels uniformly. Whilst for asymmetric noisy labels, flipping labels only occurs within a specific set of classes \cite{patrini2017making, zhang2018generalized}, for example, for MNIST, flipping $2 \to 7$, $3 \to 8$, $5 \leftrightarrow 6$ and $7 \to 1$; for CIFAR-10, flipping TRUCK $\to$ AUTOMOBILE, BIRD $\to$ AIRPLANE, DEER $\to$ HORSE, CAT $\leftrightarrow$ DOG; for CIFAR-100, the 100 classes are grouped into 20 super-classes with each has 5 sub-classes, then flipping between two randomly selected sub-classes within each super-class.

\subsection{Empirical Understanding of SL}\label{understanding_sce}
We conduct experiments on CIFAR-10 dataset with symmetric noise towards a deeper understanding of SL. 

\noindent\textbf{Experimental setup:} We use an 8-layer CNN with 6 convolutional layers followed by 2 fully connected layers. 
All networks are trained using SGD with momentum 0.9, weight decay $10^{-4}$ and an initial learning rate of 0.01 which is divided by 10 after 40 and 80 epochs (120 epochs in total). The parameter $\alpha$, $\beta$ and $A$ in SL are set to 0.1, 1 and -6 respectively.

\noindent\textbf{Class-wise learning:} The class-wise test accuracy of SL on 40\% noisy labels has already been presented in Figure \ref{sce_noisy}. Here we provide further results for 60\% noisy labels in Figure~\ref{fig:ce_sce_test_acc_60}. Under both settings, each class is more sufficiently learned by SL than CE, accompanied by accuracy increases. Particularly for the hard classes (\textit{e.g.}, classes $2/3/4/5$), SL significantly improves their learning performance.
This is because SL facilitates an adaptive pace to encourage learning from hard classes. During learning, samples from easy classes can be quickly learned to have a high probability $p_k > 0.5$, while samples from hard classes still have a low probability $p_k < 0.5$. SL will balance this discrepancy by increasing the learning speed for samples with $p_k < 0.5$ while decreasing the learning speed for those with $p_k > 0.5$. 

\begin{figure}[!t]
	\centering
	\begin{subfigure}{0.49\linewidth}
		\includegraphics[width=\textwidth]{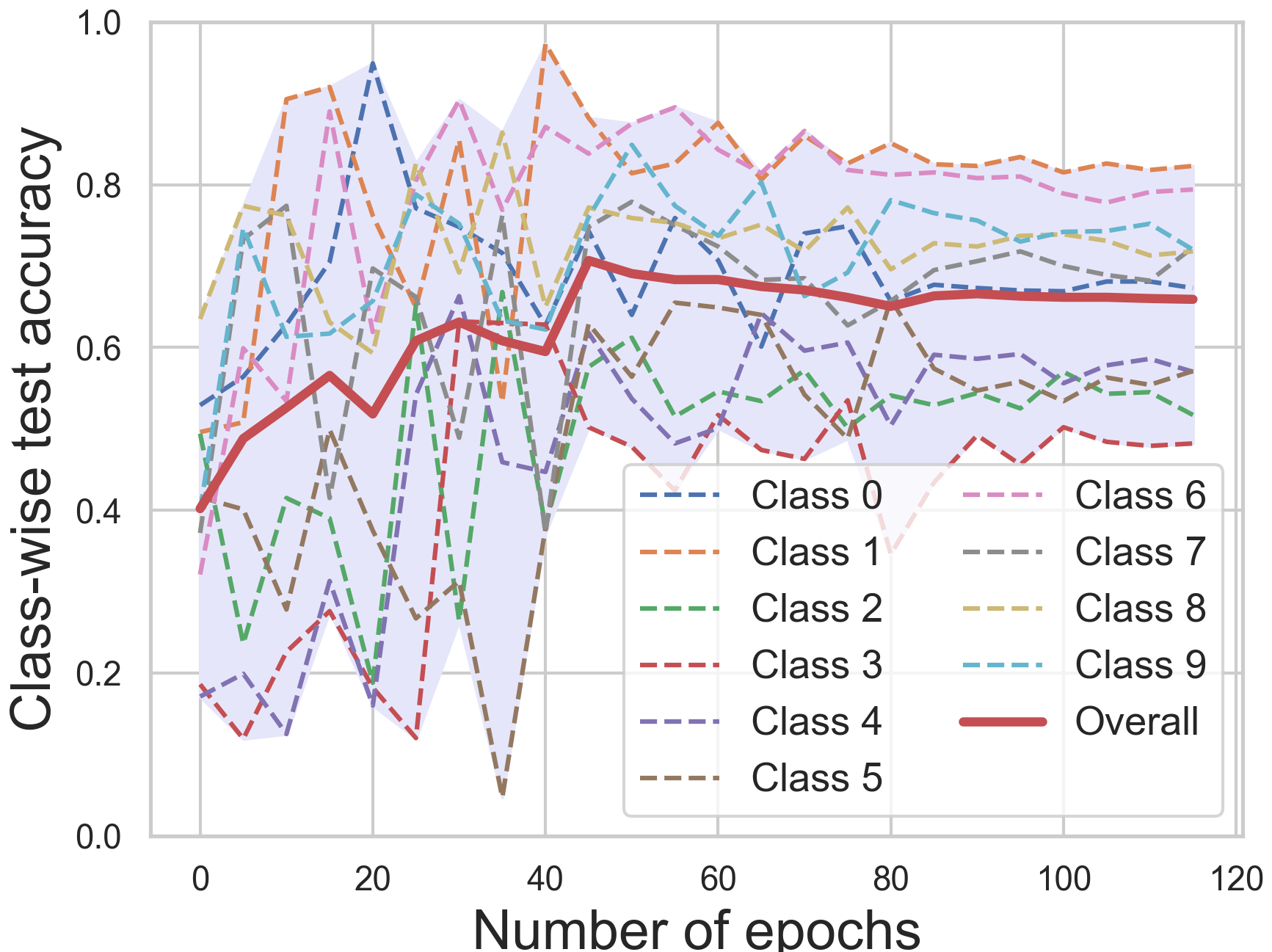}
		\caption{CE}
		\label{ce_test_acc}
	\end{subfigure}
	\begin{subfigure}{0.49\linewidth} 
		\includegraphics[width=\textwidth]{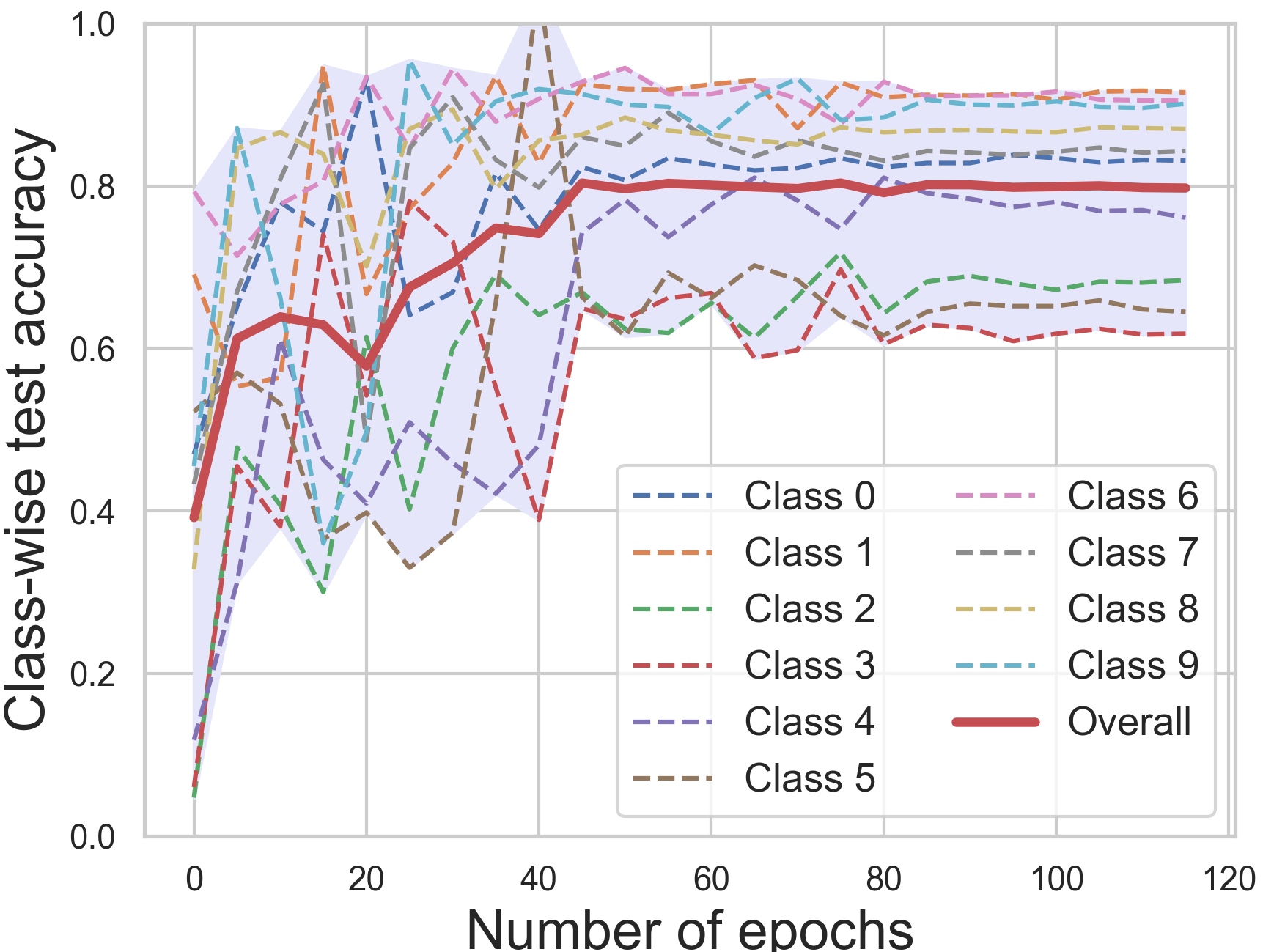}
		\caption{SL} 
		\label{sce_test_acc}
	\end{subfigure}
	\vspace{-0.1 in}
	\caption{Class-wise test accuracy of CE and SL on CIFAR-10 dataset with 60\% symmetric noisy labels. The red solid lines are the overall test accuracies.}
	\label{fig:ce_sce_test_acc_60}
	\vspace{-0.1 in}
\end{figure}

\begin{figure}[!t]
	\centering
	\begin{subfigure}{0.49\linewidth}
		\includegraphics[width=\textwidth]{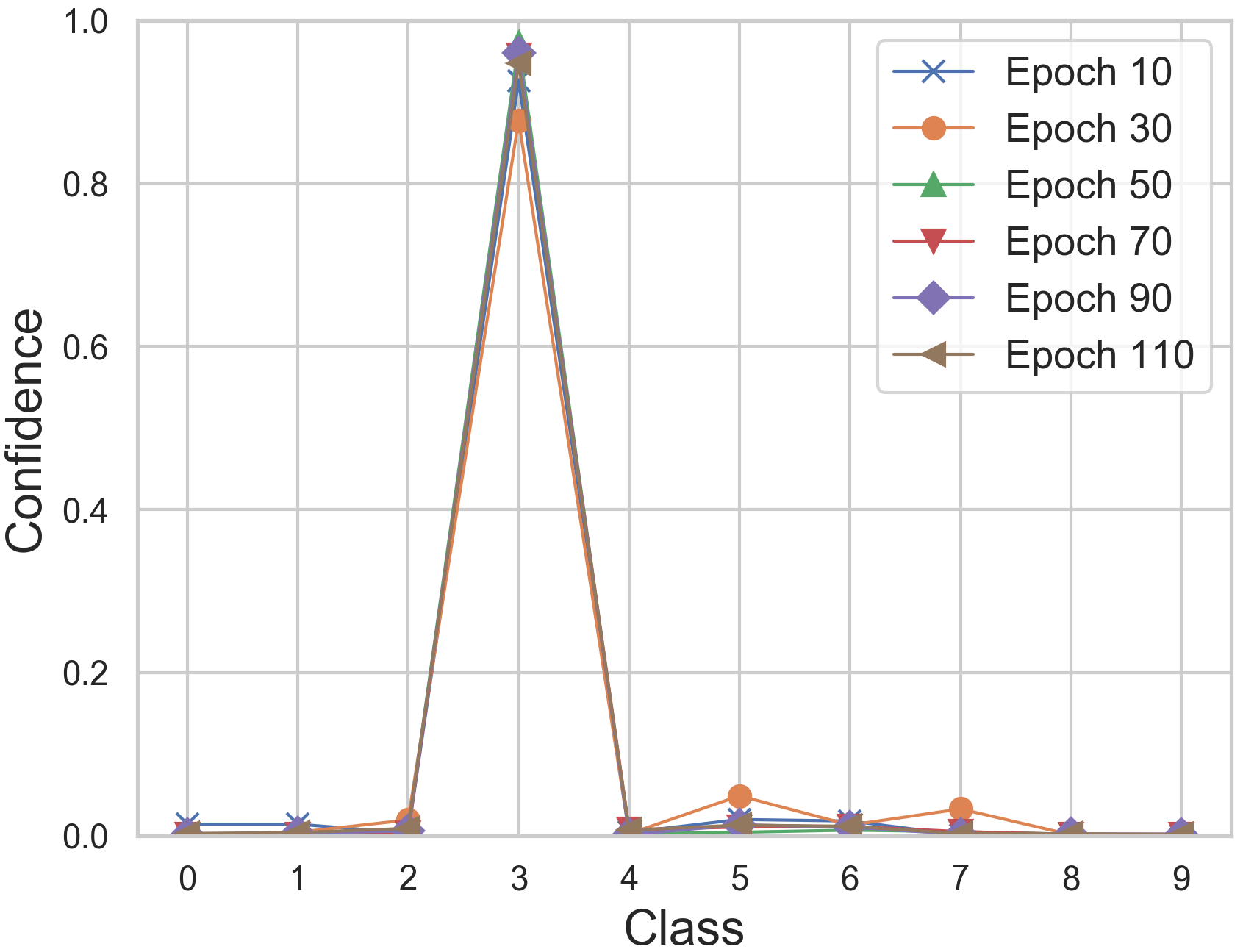}
		\caption{Prediction confidence}
		\label{sce_confidence}
	\end{subfigure}
	\begin{subfigure}{0.49\linewidth}
		\includegraphics[width=\textwidth]{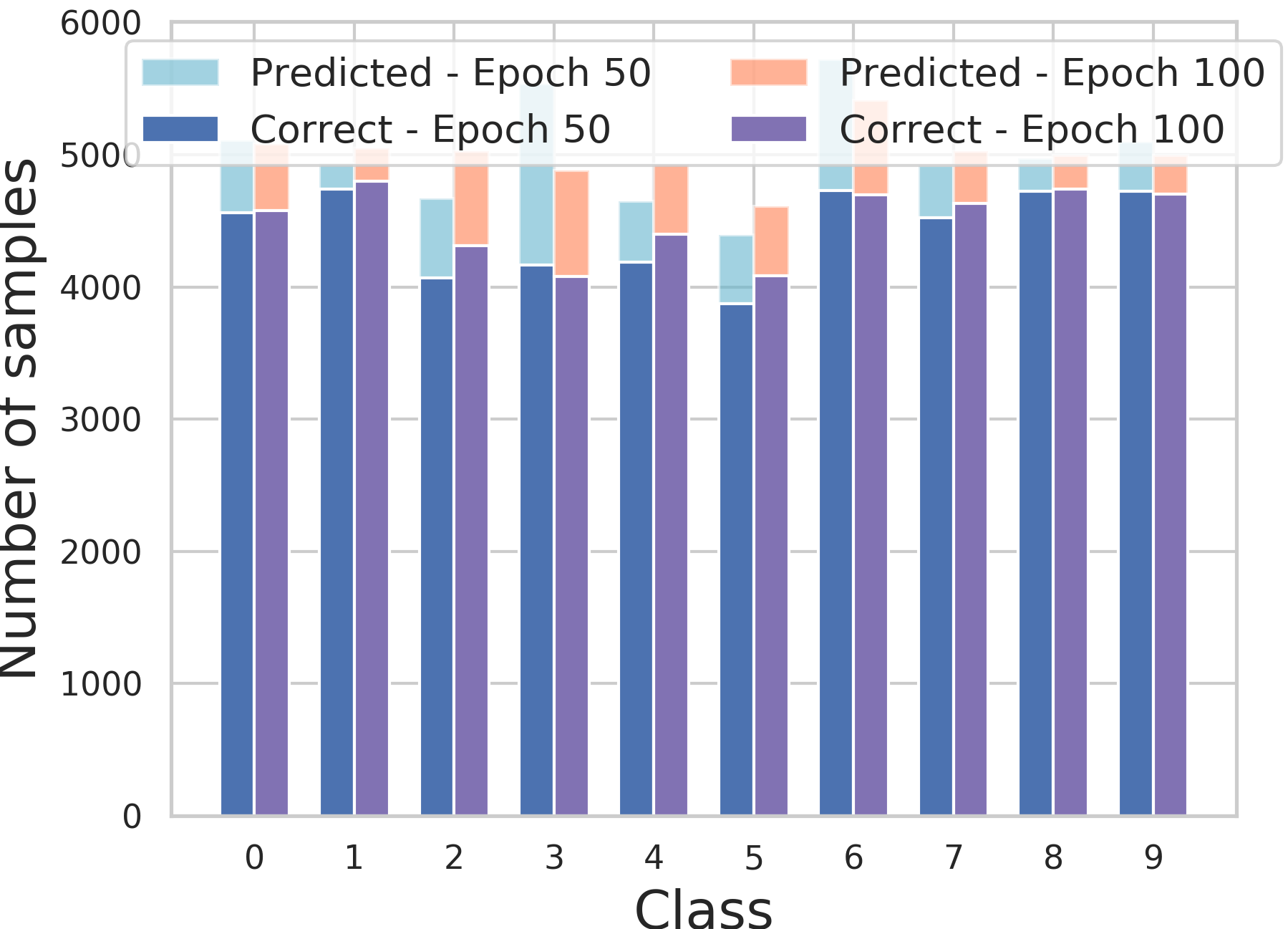}
		\caption{Prediction distribution} 
		\label{sce_prediction}
	\end{subfigure}
	\vspace{-0.1 in}
	\caption{Effect of the proposed SL on prediction confidence/distribution on CIFAR-10 with 40\% noisy labels. (a) Average confidence of the \textbf{clean portion} of class 3 samples. (b) The true positive samples (\emph{correct}) out of predictions (\emph{predicted}) for each class.}
	\label{fig:ce_sce_learning}
	\vspace{-0.1 in}
\end{figure}

\begin{figure}[!t]
	\centering
	\begin{subfigure}{0.49\linewidth}
		\includegraphics[width=\textwidth]{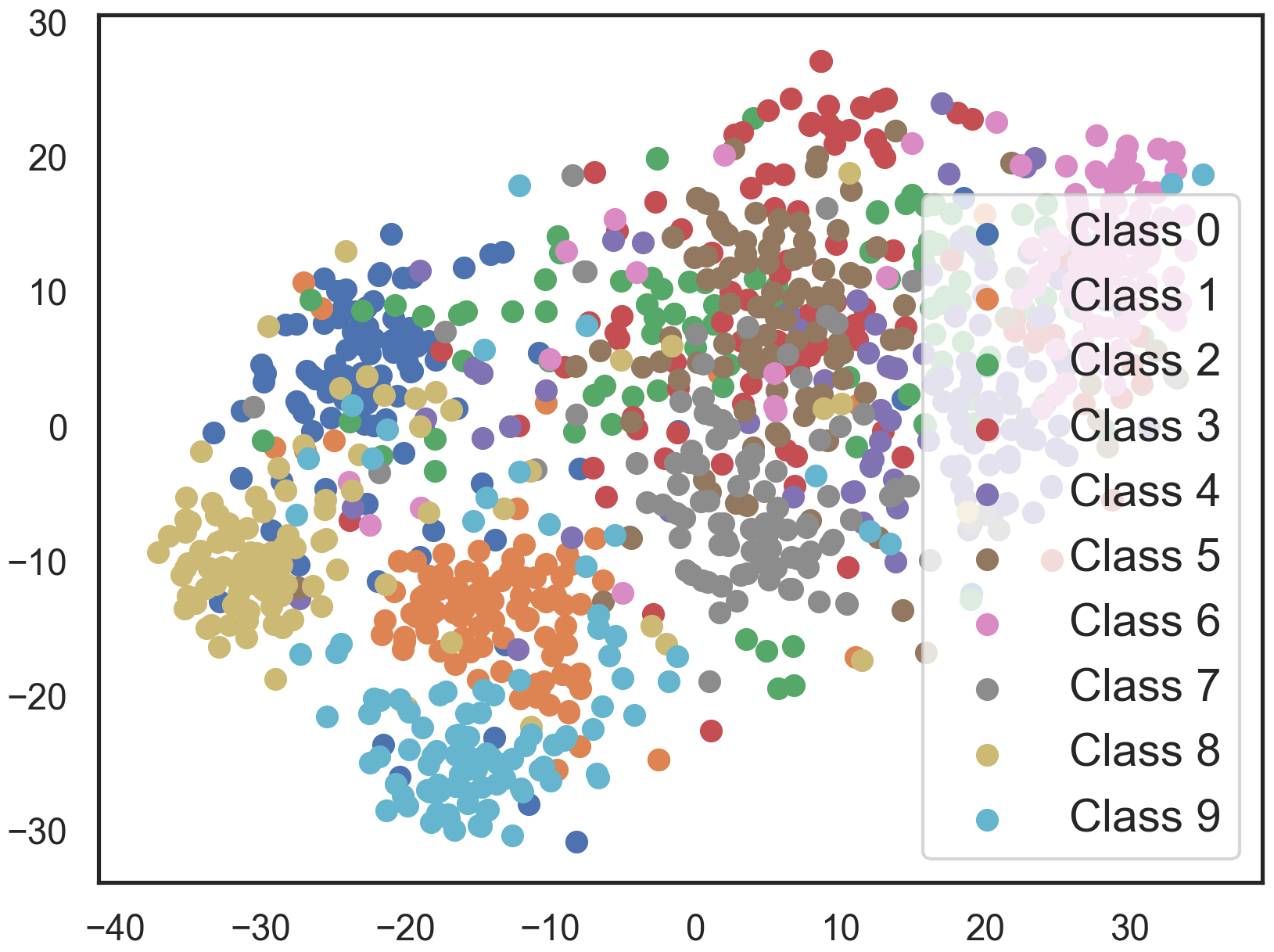}
		\caption{CE}
	\end{subfigure}
	\begin{subfigure}{0.49\linewidth} 
		\includegraphics[width=\textwidth]{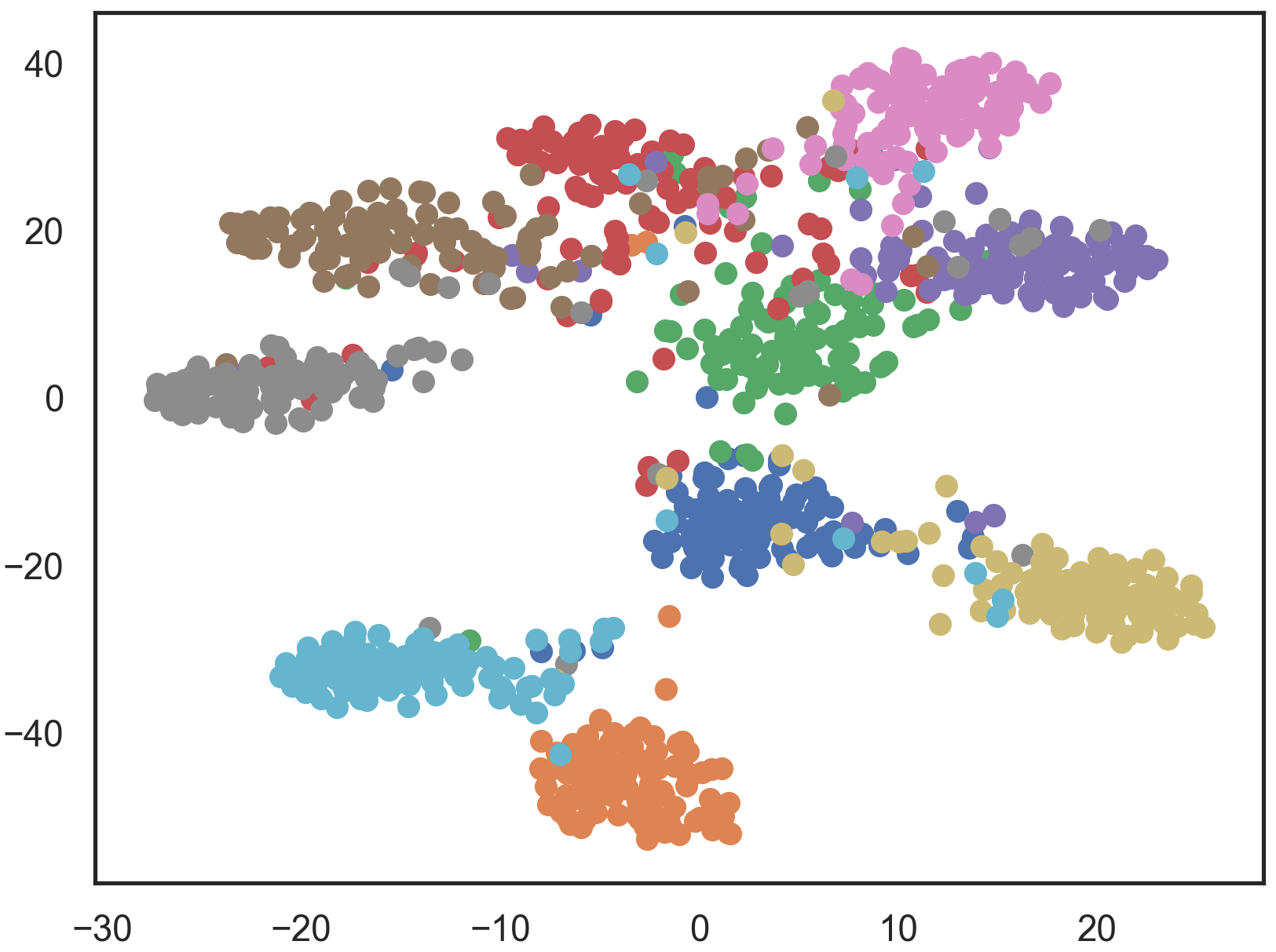}
		\caption{SL} 
	\end{subfigure}
	\vspace{-0.1 in}
	\caption{Representations learned by CE and SL on CIFAR-10 dataset with 60\% symmetric noisy labels.}
	\label{fig:rep_ce_sce_60}
	\vspace{-0.15 in}
\end{figure}

\noindent\textbf{Prediction confidence and distribution:} 
In comparison to the low confidence of CE on the clean samples in Figure \ref{ce_confidence}, we train the same network using SL under the same setting. As shown in Figure \ref{sce_confidence}, on the \textbf{clean portion} of class 3 samples, SL successfully pulls up the confidence to 0.95, while at the same time, pushes down the residual confidence at other classes to almost 0. As further shown in Figure \ref{sce_prediction}, the prediction distributions demonstrate that each class contains more than 4000 true positive samples, including the hard classes (\textit{e.g.}, class $2/3/4/5$).  Some classes (\textit{e.g.}, class $1/6/7/8/9$) even obtain close to 5000 true positive samples (the ideal case). Compared to the earlier results in Figure \ref{pred_imbalance}, SL achieves considerable improvement on each class. 

\noindent\textbf{Representations:} 
We further investigate the representations learned by SL compared to that learned by CE. We extract the high-dimensional representation at the second last dense layer, then project to a 2D embedding using t-SNE \cite{maaten2008visualizing}. The projected representations are illustrated in Figures \ref{ce_sce_rep_40} and \ref{fig:rep_ce_sce_60} for 40\% and 60\% noisy labels respectively. Under both settings, the representations learned by SL are of significantly better quality than that of CE with more separated and clearly bounded clusters.  

\noindent\textbf{Parameter analysis:}
We tune the parameters of SL: $\alpha$, $\beta$ and $A$. As $\beta$ can be reflected by $A$, here we only show results of $\alpha$ and $A$. We tested A in $[-8, -2]$ with step 2 and $\alpha \in [10^{-2}, 1]$ on CIFAR-10 under 60\% noisy labels. Figure \ref{sce_alpha} shows that large $\alpha$ (\textit{e.g.}, 1.0/0.5) tends to cause more overfitting, while small $\alpha$ (\textit{e.g.}, 0.1/0.01) can help ease the overfitting of CE. Nevertheless, the convergence can become slow when $\alpha$ is too small (\textit{e.g.}, 0.01), a behaviour like the single RCE. For this reason, a relatively large $\alpha$ can help convergence on difficult datasets such as CIFAR-100. As for parameter $A$, if the overfitting of CE is well controlled by $\alpha=0.1$, SL is not sensitive to $A$ (Figure \ref{sce_A1}). However, if CE overfitting is not properly addressed, SL becomes mildly sensitive to $A$ (Figure \ref{sce_A2}).

\begin{figure}[!t]
	\centering
	\begin{subfigure}{0.32\linewidth}
		\includegraphics[width=\textwidth]{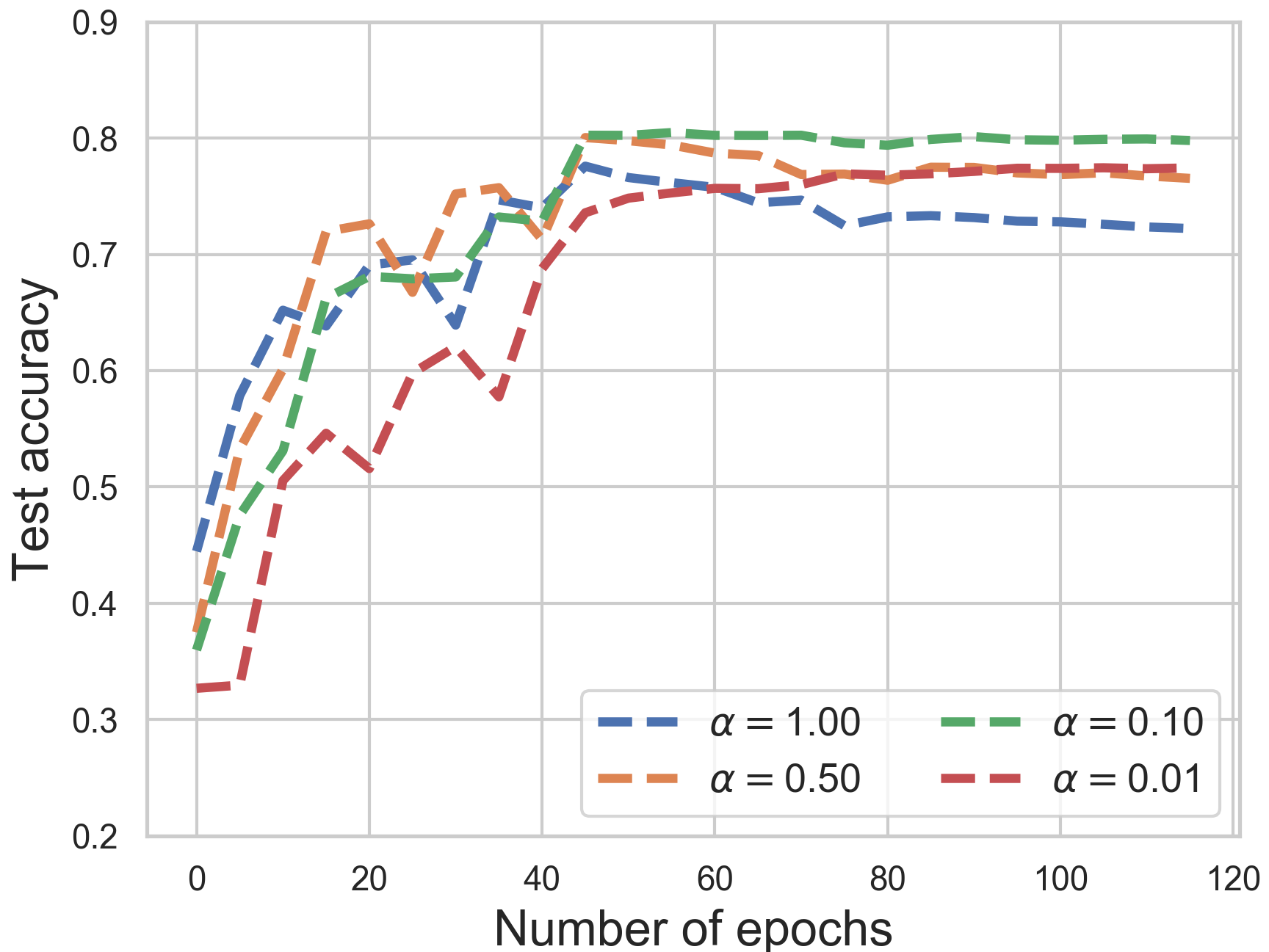}
		\caption{$\alpha$ ($A$=-6)}
		\label{sce_alpha}
	\end{subfigure}
	\begin{subfigure}{0.32\linewidth}
		\includegraphics[width=\textwidth]{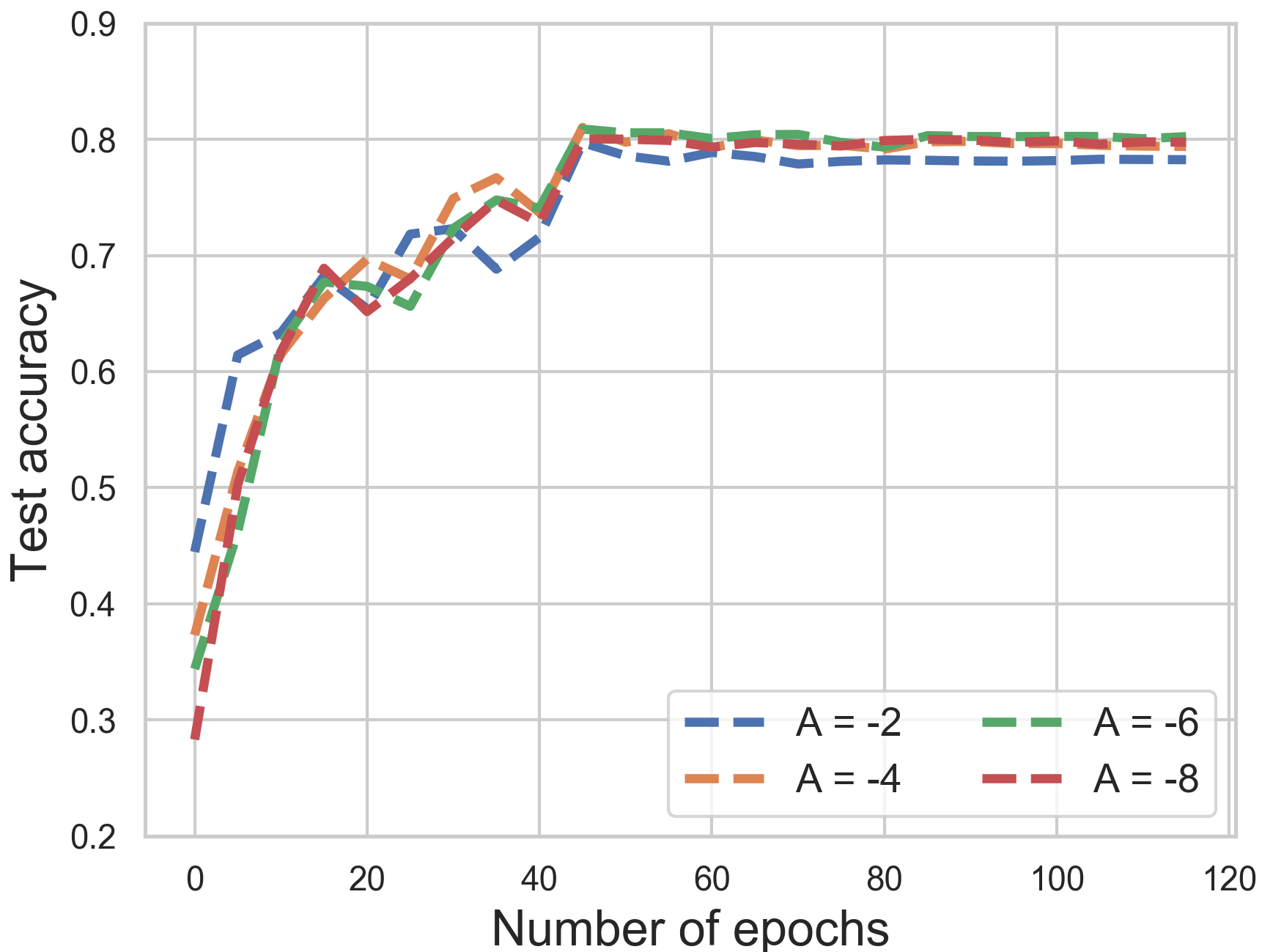}
		\caption{$A/\beta$ ($\alpha$=0.1)}
		\label{sce_A1}
	\end{subfigure}
	\begin{subfigure}{0.32\linewidth}
		\includegraphics[width=\textwidth]{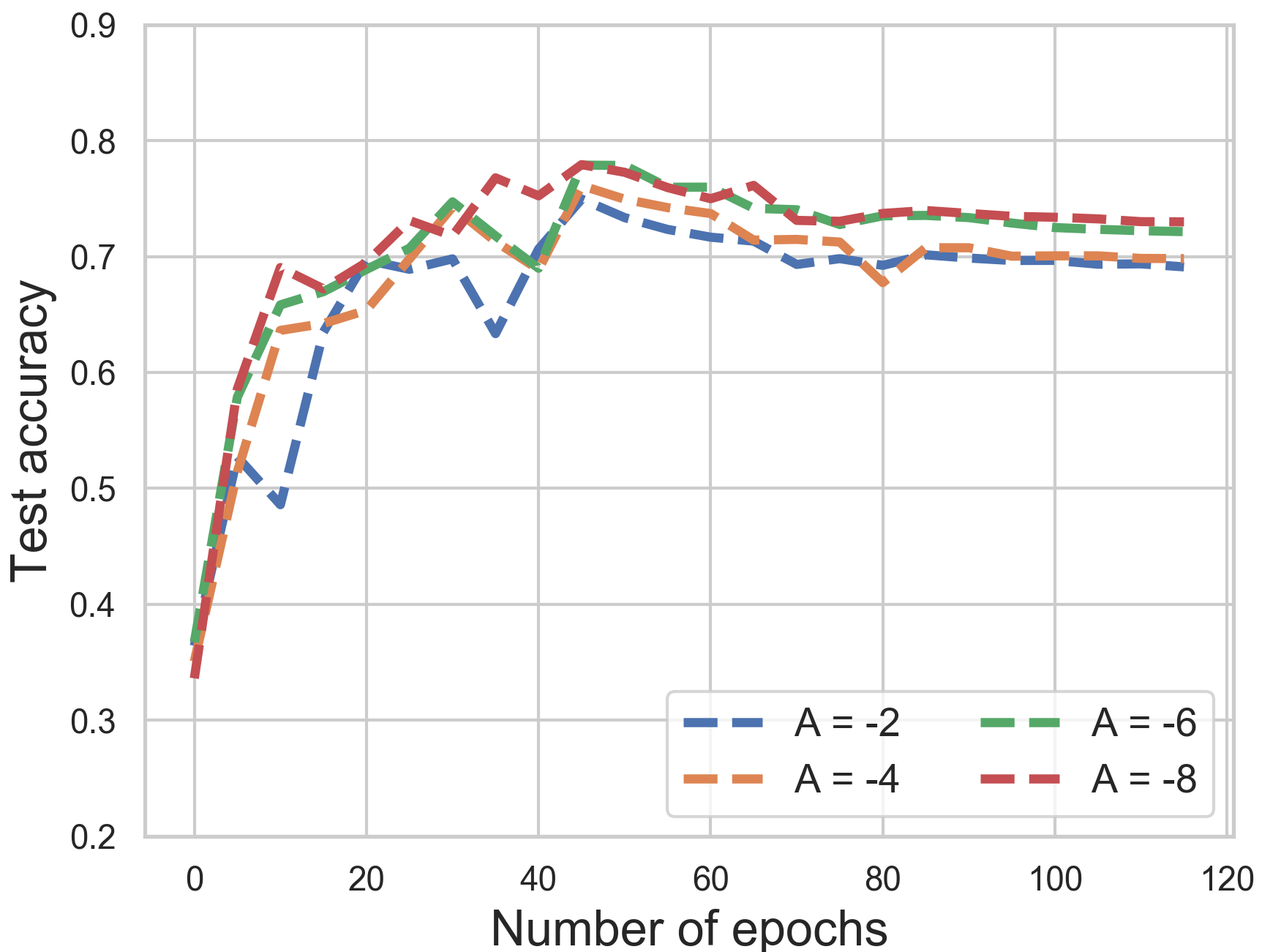}
		\caption{$A/\beta$ ($\alpha$=1)} 
		\label{sce_A2}
	\end{subfigure}
	\vspace{-0.1 in}
	\caption{Parameter analysis for SL with an 8-layer CNN on CIFAR-10 dataset under 60\% symmetric label noise: (a) Tuning $\alpha$ (fix $A$ = -6); (b) Tuning $A/\beta$ (fix $\alpha$ = 0.1); and (c) Tuning $A/\beta$ (fix $\alpha$ = 1).}
	\vspace{-0.1 in}
	\label{fig:param_analysis}
\end{figure}

\noindent\textbf{Ablation study:}
For a comprehensive understanding of each term in SL, we further conduct a series of ablation experiments on CIFAR-10 under 60\% noisy labels. Figure \ref{ablation_sce} presents the following experiments: 1) removing the RCE term; 2) removing the CE term; 3) upscaling the CE term; and 4) upscaling the RCE term. We can observe that simply upscaling CE does not help learning, or even leads to more overfitting. The RCE term itself does not exhibit overfitting even when upscaled, but it converges slowly. But when CE and RCE are combined into the SL framework, the performance is drastically improved. 

\begin{figure}[!t]
	\centering
	\begin{subfigure}{0.49\linewidth}
		\includegraphics[width=\textwidth]{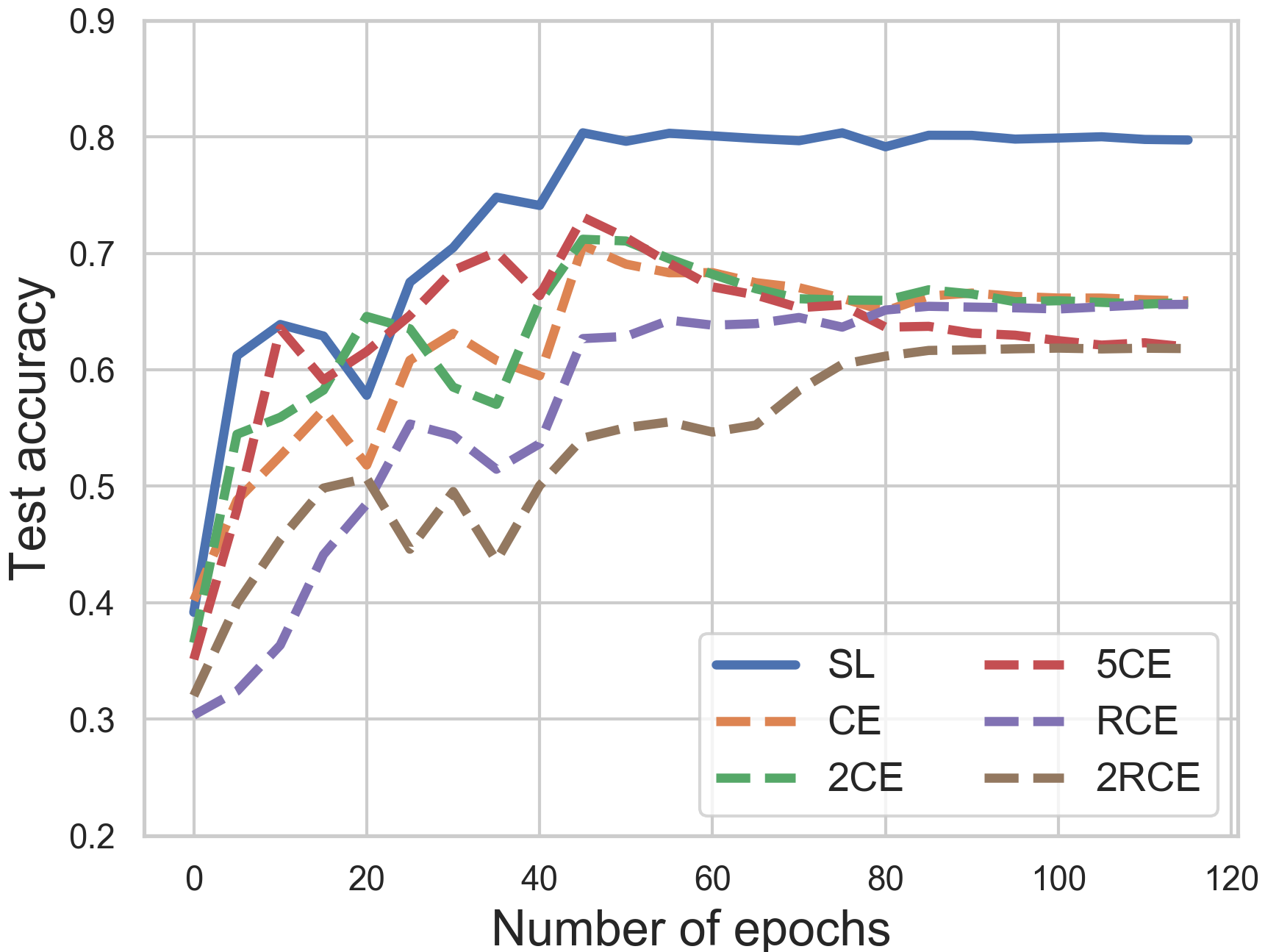}
		\caption{Ablation of SL}
		\label{ablation_sce}
	\end{subfigure}
	\begin{subfigure}{0.49\linewidth} 
		\includegraphics[width=\textwidth]{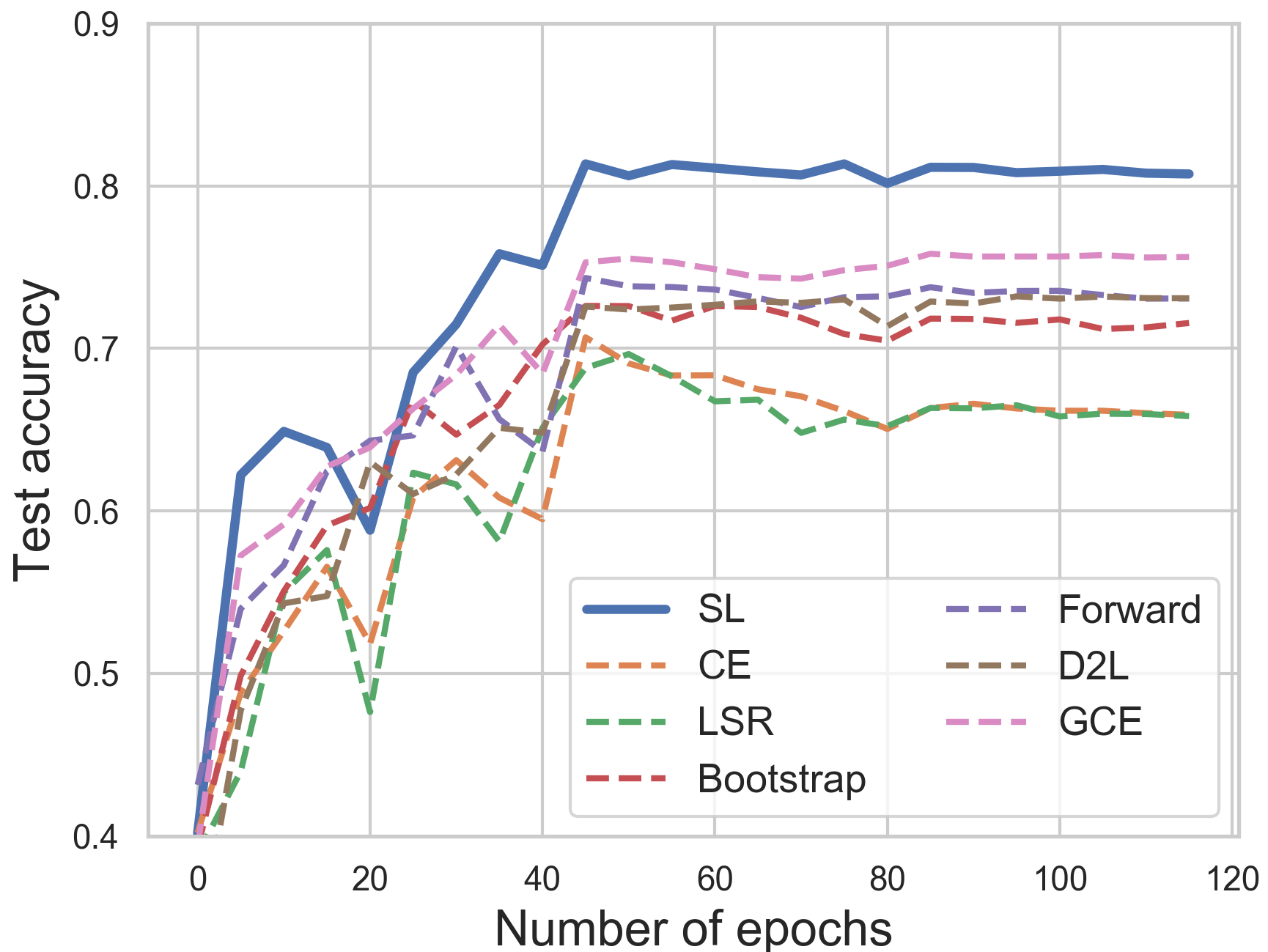}
		\caption{SL vs. baselines} 
		\label{learning_curve}
	\end{subfigure}
	\vspace{-0.1 in}
	\caption{Accuracy of different models on CIFAR-10 with 60\% symmetric label noise. (a) Ablation study of SL; (b) Comparison between SL and other baselines.}
	\vspace{-0.15 in}
	\label{fig:ce_sce_rep_40_asym}
\end{figure}

\begin{table*}[!t]
\caption{Test accuracy (\%) of different models on benchmark datasets with various rates of symmetric and asymmetric noisy labels. The average accuracy and standard deviation of 5 random runs are reported and the best results are in \textbf{bold}.}
\vspace{-0.1 in}
\label{tab:all_acc}
\centering
\begin{adjustbox}{width=1\textwidth}
\small
\begin{tabular}{c|c|ccccc|ccc}
\hline
\multirow{3}{*}{Datasets} & \multirow{3}{*}{Methods} & \multicolumn{5}{c}{Symmetric Noise} & \multicolumn{3}{c}{Asymmetric Noise} \\ \cline{3-10}
 &  & \multicolumn{5}{c}{Noise Rate $\eta$} & \multicolumn{3}{c}{Noise Rate $\eta$} \\
& & 0.0 & 0.2 & 0.4 & 0.6 & 0.8 & 0.2 & 0.3 & 0.4 \\ \hline \hline
\multirow{7}{*}{\begin{tabular}[c]{@{}c@{}}\\MNIST\\ \end{tabular}} & CE & $ 99.02 \pm 0.01$ & $88.71 \pm 0.05$ & $ 69.56 \pm 0.19 $ & $ 46.54 \pm 0.28$ & $ 21.77 \pm 0.07$ & $ 93.14 \pm 0.04$ & $87.91 \pm 0.05$  & $ 81.10 \pm 0.07$ \\
 & LSR & $99.28 \pm 0.01$ & $ 89.56 \pm 0.06$ & $ 68.11 \pm 0.24$ & $ 45.01 \pm 0.15$ & $ 21.28 \pm 0.27$ & $ 94.18 \pm 0.08$ & $88.39 \pm 0.20$ & $81.09 \pm 0.35$ \\
 & Bootstrap & $99.08 \pm 0.01$ & $ 88.72 \pm 0.14$ & $ 69.97 \pm 0.36$ & $ 47.06 \pm 0.26$ & $ 22.60 \pm 0.27$ & $ 93.31 \pm 0.03$ & $87.87 \pm 0.09$ & $80.46 \pm 0.15$ \\
 & Forward & $ 99.03 \pm 0.01$ & $ 94.85 \pm 0.07$ & $ 86.02 \pm 0.13$ & $69.77 \pm 0.41$ & $ 49.72 \pm 0.30$ & $97.31 \pm 0.05$ & $96.25 \pm 0.10$ & $95.72 \pm 0.09$ \\
 & D2L & 99.27 $\pm$ 0.01 & 98.80 $\pm$ 0.01 & 98.49 $\pm$ 0.01 & ${93.61 \pm 0.01}$ & 48.57 $\pm$ 0.04 & 98.71 $\pm$ 0.02 & $97.77 \pm 0.04$ & 93.32 $\pm$ 0.15 \\
 & GCE & $ 99.04 \pm 0.01$ & $ 98.66 \pm 0.01$ & $ 97.17 \pm 0.01$ & $ 79.65 \pm 0.14$ & $31.55 \pm 0.18$ & $96.73 \pm 0.08$ & $88.46 \pm 0.18$ & $ 81.26 \pm 0.11$  \\
& \textbf{SL} & $ \boldsymbol{99.32 \pm 0.01}$ & $ \boldsymbol{99.02 \pm 0.01}$ & $ \boldsymbol{98.97 \pm 0.01}$ & $\boldsymbol{97.40 \pm 0.02}$ & $\boldsymbol{65.02 \pm 0.19}$ & $\boldsymbol{99.18 \pm 0.01}$ & $\boldsymbol{98.85 \pm 0.01}$ & $ \boldsymbol{98.00 \pm 0.02}$ \\\hline
\hline
\multirow{7}{*}{CIFAR-10} & CE & 89.26 $\pm$ 0.03 & 82.96 $\pm$ 0.05 & 78.70 $\pm$ 0.07 & 66.62 $\pm$ 0.15 & 34.80 $\pm$ 0.25 & 85.98 $\pm$ 0.03 & 83.53 $\pm$ 0.08 & 78.51 $\pm$ 0.05 \\
& LSR & 88.57 $\pm$ 0.04 & 83.49 $\pm$ 0.05 & 78.41 $\pm$ 0.03 & 67.38 $\pm$ 0.15 & 36.30 $\pm$ 0.16 & 85.38 $\pm$ 0.05 & 82.89 $\pm$ 0.12 & 77.88 $\pm$ 0.20 \\
 & Bootstrap & 88.77 $\pm$ 0.06 & 83.95 $\pm$ 0.10 & 79.97 $\pm$ 0.07 & 71.65 $\pm$ 0.05 & 41.44 $\pm$ 0.49 & 86.57 $\pm$ 0.08 & 84.86 $\pm$ 0.05 & 79.76 $\pm$ 0.07 \\
 & Forward & 89.39 $\pm$ 0.04 & 85.83 $\pm$ 0.05 & 81.37 $\pm$ 0.03 & 73.59 $\pm$ 0.08 & 47.10 $\pm$ 0.14 & 87.68 $\pm$ 0.01 & $\boldsymbol{86.86 \pm 0.06}$ & $\boldsymbol{85.73 \pm 0.04}$ \\
 & D2L & 86.66 $\pm$ 0.05 & 81.13 $\pm$ 0.06 & 76.80 $\pm$ 0.12 & 60.67 $\pm$ 0.12 & 19.83 $\pm$ 0.05 & 82.72 $\pm$ 0.06 & 80.41 $\pm$ 0.05 & 73.33 $\pm$ 0.12 \\
 & GCE & 86.76 $\pm$ 0.03 & 84.86 $\pm$ 0.06 & 82.42 $\pm$ 0.10 & 75.20 $\pm$ 0.09 & 40.81 $\pm$ 0.24 & 84.61 $\pm$ 0.09 & 82.11 $\pm$ 0.13 & 75.32 $\pm$ 0.10 \\
 & \textbf{SL} & $\boldsymbol{89.28 \pm 0.04}$ & $\boldsymbol{87.63 \pm 0.06}$ & $\boldsymbol{85.34 \pm 0.07}$ & $\boldsymbol{80.07 \pm 0.02}$ & $\boldsymbol{53.81 \pm 0.27}$ & $\boldsymbol{88.24 \pm 0.05}$ & 85.36 $\pm$ 0.14 & 80.64 $\pm$ 0.10 \\ \hline
\hline
\multirow{7}{*}{CIFAR-100} & CE & 64.34 $\pm$ 0.37 & 59.26 $\pm$ 0.39 & 50.82 $\pm$ 0.19 & 25.39 $\pm$ 0.09 & 5.27 $\pm$ 0.06 & 62.97 $\pm$ 0.19 & 63.12 $\pm$ 0.16 & 61.85 $\pm$ 0.35 \\
 & LSR & 63.68 $\pm$ 0.54 & 58.83 $\pm$ 0.40 & 50.05 $\pm$ 0.31 & 24.68 $\pm$ 0.43 & $5.22 \pm 0.07$ & $63.03 \pm 0.48$ & $62.32 \pm 0.48$ & $61.59 \pm 0.41$ \\
 & Bootstrap & 63.26 $\pm$ 0.39 & 57.91 $\pm$ 0.42 & 48.17 $\pm$ 0.18 & 12.27 $\pm$ 0.11 & $1.00 \pm 0.01$ & $63.44 \pm 0.35$ & $63.18 \pm 0.35$ & $62.08 \pm 0.22$ \\
 & Forward & 63.99 $\pm$ 0.52 & 59.75 $\pm$ 0.34 & 53.13 $\pm$ 0.28 & 24.70 $\pm$ 0.26 & $2.65 \pm 0.03$ & $64.09 \pm 0.61$ & $64.00 \pm 0.32$ & $60.91 \pm 0.36$\\
 & D2L & $64.60 \pm 0.31$ & $59.20 \pm 0.43$ & $52.01 \pm 0.37$ & $35.27 \pm 0.28$ & $5.33 \pm 0.54$ & $62.43 \pm 0.28$ & $63.20 \pm 0.27$ & $61.35 \pm 0.66$ \\
 & GCE & $64.43 \pm 0.20$ & $59.06 \pm 0.27$ & $53.25 \pm 0.65$ & $36.16 \pm 0.74$ & $
 8.43 \pm 0.80$ & $63.03 \pm 0.22$ & $63.17 \pm 0.26$ & $61.69 \pm 1.15$ \\
 & \textbf{SL} & $\boldsymbol{66.75 \pm 0.04}$ & $\boldsymbol{60.01 \pm 0.19}$ & $\boldsymbol{53.69 \pm 0.07}$ & $\boldsymbol{41.47 \pm 0.04}$ & $\boldsymbol{15.00 \pm 0.04}$ & $\boldsymbol{65.58 \pm 0.06}$ & $\boldsymbol{65.14 \pm 0.05}$ & $\boldsymbol{63.10 \pm 0.13}$
\\ \hline
\end{tabular}
\end{adjustbox}
\label{result_table}
\vspace{-0.15 in}
\end{table*}

\subsection{Robustness to Noisy Labels}\label{benckmark_robust}

\noindent\textbf{Baselines:} We compare SL with 5 recently proposed noisy label learning methods as well as the standard CE loss: (1) Forward \cite{patrini2017making}: Training with label correction by multiplying the network prediction with the ground truth noise matrix; (2) Bootstrap \cite{reed2014training}: Training with new labels generated by a convex combination of the raw labels and the predicted labels; (3) GCE \cite{zhang2018generalized}: Training with a noise robust loss encompassing both MAE and CE; (4) D2L \cite{ma2018dimensionality}: Training with subspace dimensionality adapted labels; (5) Label Smoothing Regularization (LSR) \cite{pereyra2017regularizing}: Training with CE on soft labels, rather than the one-hot labels; and (6) CE: Training with standard cross entropy loss.

\noindent\textbf{Experimental setup:} 
Experiments are conducted on MNIST \cite{lecun1998gradient}, CIFAR-10 \cite{krizhevsky2009learning} and CIFAR-100 \cite{krizhevsky2009learning}. We use a 4-layer CNN for MNIST, the same network as Section \ref{understanding_sce} for CIFAR-10 and a ResNet-44 \cite{he2016deep} for CIFAR-100. Parameters for the baselines are configured according to their original papers. For our SL, we set $A = -4$ for all datasets, and $\alpha = 0.01, \beta = 1.0$ for MNIST, $\alpha = 0.1, \beta = 1.0$ for CIFAR-10, $\alpha = 6.0, \beta = 0.1$ for CIFAR-100 (a dataset known for hard convergence)\footnote{For 40\% asymmetric noise, $\beta$ is set to 5.0 for CIFAR-10 and $\alpha$ is set to 2.0 for CIFAR-100. Other parameters are unchanged.}. All networks are trained using SGD with momentum 0.9, weight decay $5 \times 10^{-3}$ and an initial learning rate of 0.1. The learning rate is divided by 10 after 10 and 30 epochs for MNIST (50 epochs in total), after 40 and 80 epochs for CIFAR-10 (120 epochs in total), and after 80 and 120 epochs for CIFAR-100 (150 epochs in total). Simple data augmentation techniques (width/height shift and horizontal flip) are applied on CIFAR-10 and CIFAR-100. For symmetric noise, we test varying noise rates $\eta \in [0\%, 80\%]$, while for asymmetric noise, we test noise rates $\eta \in [0\%, 40\%]$. 

\noindent\textbf{Robustness performance:} The classification accuracies are reported in Table \ref{tab:all_acc}. As can be seen, SL improves on the baselines via a large margin for almost all noise rates and all datasets. Note that Forward sometimes also delivers a relatively good performance, as we directly provide it with the ground truth noise matrix. We also find that SL can be more effective than GCE, particularly for high noise rates. The complete learning procedures of SL and baselines on CIFAR-10 are illustrated in Figure~\ref{learning_curve}. SL shows a clear advantage over other methods, especially in the later stages of learning with noisy labels. This is likely because that, in the later stages of DNN learning, other methods all suffer to some extent from under learning on hard classes, while SL ensures sufficient learning on them. 

\noindent\textbf{Enhancing existing methods with SL:}
We introduce some general principles to incorporate SL into existing methods to further enhance their performance. For methods that use robust loss functions or label corrections, the RCE term of SL can be directly added to the loss function, while for methods that still use the standard CE loss without label corrections, SL can be used with small $\alpha$ and large $\beta$ to replace the existing loss function. This is to avoid overfitting while promote sufficient learning. As a proof-of-concept, we conduct experiments to enhance Forward and LSR with SL. For ``Forward+SL'', we add the RCE term to the Forward loss with $\beta = 1.0/0.1$ for symmetric/asymmetric noise respectively, while for ``LSR+SL'', we use the SL loss with the same setting in Table \ref{tab:all_acc}. Results on CIFAR-10 are presented in Table \ref{tab:sce_d2l}. Both the enhanced methods demonstrate a clear performance improvement over their original versions (Forward or LSR) both on symmetric and asymmetric noise. However, in some scenarios, the enhanced methods are still not as good as SL. This often occurs when there is a large performance gap between the original methods and SL. We believe that with more adaptive incorporation and careful parameter tuning, SL can be combined with existing approaches to achieve even better performance.

\begin{table}[!t]
    \small
    \centering
    \caption{Accuracy (\%) of SL-boosted Forward, D2L and LSR methods on CIFAR-10 under various label noise.}
    \vspace{-0.1 in}
    \label{tab:sce_d2l}
    \begin{adjustbox}{width=0.49\textwidth}
    \begin{tabular}{l|cc|c}
    \hline
    \multirow{2}{*}{Method}  
    & \multicolumn{2}{c|}{Symmetric noise}  & \multicolumn{1}{c}{Asymmetric noise} \\
     & 0.4 & 0.6  & 0.4  \\
    \hline
    Forward+SL & $84.54 \pm 0.03$ & $79.64 \pm 0.04$ & $86.22 \pm 0.18$ \\
    LSR+SL & $85.20 \pm 0.01$ & $ 79.28 \pm 0.05 $ & $ 80.99 \pm $ 0.30\\
    \hline
    \end{tabular}
    \end{adjustbox}
    \vspace{-0.15 in}
\end{table}

\subsection{Experiments on Real-world Noisy Dataset}\label{sec:Clothing1M}
In the above experiments, we have seen that SL achieves excellent performance on datasets with manually corrupted noisy labels. Next, we assess its applicability for a real-world large-scale noisy dataset: Clothing1M \cite{xiao2015learning}. 

The Clothing1M dataset contains 1 million images of clothing obtained from online shopping websites with 14 classes: T-shirt, Shirt, Knitwear, Chiffon, Sweater, Hoodie, Windbreaker, Jacket, Down Coat, Suit, Shawl, Dress, Vest, and Underwear. The labels are generated by the surrounding text of images and are thus extremely noisy. The overall accuracy of the labels is ${\small\sim61.54\%}$, with some pairs of classes frequently confused with each other (\textit{e.g.}, Knitwear and Sweater), which may contain both symmetric and asymmetric label noise. The dataset also provides $50k$, $14k$, $10k$ manually refined clean data for training, validation and testing respectively, but we did not use the $50k$ clean data. The classification accuracy on the $10k$ clean testing data is used as the evaluation metric. 

\noindent\textbf{Experimental setup:} 
We use ResNet-50 with ImageNet pretrained weights  similar to \cite{patrini2017making,xiao2015learning}. For preprocessing, images are resized to $256 \times 256$, with mean value subtracted and cropped at the center of $224 \times 224$. We train the models with batch size 64 and initial learning rate $10^{-3}$, which is reduced by $1/10$ after 5 epochs (10 epochs in total). SGD with a momentum 0.9 and weight decay $10^{-3}$ are adopted as the optimizer. Other settings are the same as Section \ref{benckmark_robust}. 

\noindent\textbf{Results:} 
As shown in Table \ref{tab:clothing}, SL obtains the highest performance compared to the baselines.  We also find that Forward achieves a relatively good result, though it requires the use of the part of data that both have noisy and clean labels to obtain the noise transition matrix, which is not often available in real-world settings. SL only requires the noisy data and does not require extra auxiliary information.

\begin{table}[!t]
\centering
\small
\caption{Accuracy (\%) of different models on real-world noisy dataset Clothing1M. The best results are in \textbf{bold}.}
\vspace{-0.1 in}
\label{tab:clothing}
\begin{adjustbox}{width=0.49\textwidth}
\begin{tabular}{l|cccccc}
\hline
Methods & CE & Bootstrap & Forward & D2L & GCE & \textbf{SL}\\ \hline
Acc & 68.80 & 68.94 & 69.84 & 69.47 & 69.75 & $\bm{71.02}$\\
\hline
\end{tabular}
\end{adjustbox}
\vspace{-0.15 in}
\end{table}

\section{Conclusions}\label{sec:conclusion}
In this paper, we identified a deficiency of cross entropy (CE) used in DNN learning for noisy labels, in relation to under learning of hard classes. To address this issue, we proposed the Symmetric cross entropy Learning (SL), boosting CE symmetrically with the noise robust Reverse Cross Entropy (RCE), to simultaneously addresses its under learning and overfitting problems. We provided both theoretical and empirical understanding on SL, and demonstrated its effectiveness against various types and rates of label noise on both benchmark and real-world datasets. Overall, due to its simplicity and ease of implementation, we believe SL is a promising loss function for training robust DNNs against noisy labels, and an attractive framework to be used along with other techniques for datasets containing noisy labels.

{
\small
\bibliographystyle{ieee_fullname}
\bibliography{egbib}
}

\newpage
\onecolumn
\appendix
\setcounter{theorem}{0}

\section{Proof for Theorem 1}\label{appendix_proof}

\begin{theorem}
In a multi-class classification problem, $\ell_{rce}$ is noise tolerant under symmetric or uniform label noise if noise rate $\eta < 1 -\frac{1}{K}$. And, if $R(f^*) = 0$, $\ell_{rce}$ is also noise tolerant under asymmetric or class-dependent label noise when noise rate $\eta_{y k} < 1-\eta_{y}$ with $\sum_{k \neq y}\eta_{y k}=\eta_{y}$.
\end{theorem}
\begin{proof}
For symmetric noise:
\begin{align*}
	\small
	R^\eta(f) & =  \E_{\xx, \hat{y}} \ell_{rce}(f(\xx), \hat{y}) =  \E_{\xx} \E_{y | \xx} \E_{\hat{y} | \xx, y} \ell_{rce}(f(\xx), \hat{y}) \\
		&= \E_{\xx} \E_{y | \xx} \Big[ (1-\eta) \ell_{rce}(f(\xx), y) + \frac{\eta}{K-1} \sum_{k\neq y} \ell_{rce}(f(\xx), k) \Big] \\
		& =  (1 - \eta) R(f) +  \frac{\eta}{K-1} \bigg(\E_{\xx, y}\bigg[\sum_{k=1}^{K}\ell_{rce}(f(\xx), k)\bigg] - R(f)\bigg)\\
		& = R(f)\left(1-\frac{\eta K}{K-1}\right) - A\eta,
\end{align*}
where the last equality holds due to $\sum_{k=1}^{K}\ell_{rce}(f(\xx), k) = - (K-1)A$ following Eq. \eqref{eq:rce} and the definition of $\log 0 = A$ (a negative constant). Thus, 
	\[R^\eta(f^*)-R^\eta(f)=(1-\frac{\eta K}{K-1})(R(f^*)-R(f)) \leq 0,\]
because $\eta < 1 - \frac{1}{K}$ and $f^*$ is a global minimizer of $R(f)$. This proves $f^*$ is also the global minimizer of risk $R^\eta(f)$, that is, $\ell_{rce}$ is noise tolerant. 
	
For asymmetric or class-dependent noise, $1-\eta_y$ is the probability of a label being correct (\textit{i.e.,} $k=y$), and the noise condition $\eta_{y k} < 1-\eta_{y}$ generally states that a sample $\xx$ still has the highest probability of being in the correct class $y$, though it has probability of $\eta_{y k}$ being in an arbitrary noisy (incorrect) class $k \neq y$. Considering the noise transition matrix between classes $[\eta_{i j}], \forall i,j \in \{1,2, \cdots, K\}$, this condition only requires that the matrix is diagonal dominated by $\eta_{i i}$ (\textit{i.e.,} the correct class probability $1 - \eta_{y}$). Following the symmetric case, here we have,
	\small{
	\begin{align}
	\label{eq:cc_1}
	\begin{split}
	R^\eta(f) & = \E_{\xx, \hat{y}} \ell_{rce}(f(\xx), \hat{y}) =  \E_{\xx} \E_{y | \xx} \E_{\hat{y} | \xx, y} \ell_{rce}(f(\xx), \hat{y}) \\ 
	& = \E_{\xx} \E_{y | \xx} \Big[ (1- \eta_{y}) \ell_{rce}(f(\xx), y) +   \sum_{k \neq y} \eta_{y k} \ell_{rce} (f(\xx), k) \Big] \\
	& = \E_{\xx, y} \Big[ (1-\eta_{y})\Big(\sum_{k=1}^{K}\ell_{rce}(f(\xx), k) - \sum_{k \neq y} \ell_{rce}(f(\xx), k)\Big)\Big] + \E_{\xx, y} \Big[\sum_{k \neq y} \eta_{y k} \ell_{rce}(f(\xx), k)\Big] \\
	& = \E_{\xx, y} \Big[ (1-\eta_{y})\big(-(K-1)A - \sum_{k \neq y} \ell_{rce}(f(\xx), k)\big)\Big] + \E_{\xx, y}\Big[ \sum_{k \neq y} \eta_{y k} \ell_{rce}(f(\xx), k)\Big] \\
	& = - (K-1)A \E_{\xx, y} (1-\eta_{y})-\E_{\xx, y} \Big[\sum_{k \neq y}(1-\eta_{y}-\eta_{y k}) \ell_{rce}(f(\xx), k)\Big].
	\end{split}
	\end{align}	}
	As $f^{\ast}_{\eta}$ is the minimizer of $R^\eta(f)$, $R^{\eta}(f_{\eta}^{\ast})-R^{\eta}(f^{\ast}) \leq  0$. So, from Eq.(\ref{eq:cc_1}), we have,
	\begin{align}
	\label{eq:cc_2}
	\E_{\xx, y}\Big[\sum_{k\neq y}(1-\eta_{y}-\eta_{y k})\big(\underbrace{\ell_{rce} (f^{\ast}(\xx),k)}_{\ell_{rce}^{*}}-\underbrace{\ell_{rce}(f^{\ast}_{\eta}(\xx),k)}_{\ell_{rce}^{\eta *}}\big)\Big] \leq 0.
	\end{align}
	Next, we prove, $f^{\ast}_{\eta}=f^{\ast}$ holds following Eq. (\ref{eq:cc_2}). First, $(1-\eta_{y}-\eta_{y k})>0$ as per the assumption that $\eta_{y k} < 1-\eta_{y}$. Since we are given $R(f^*)=0$, we have $\ell_{rce}(f^*(\xx),k)= -A$ for all $k\neq y$. Also, by the definition of $\ell_{rce}^{\eta*}$, we have $\ell_{rce}(f_\eta^*(\xx) , k) =-A(1-p_k)\leq -A$, $\forall k\neq y$. Thus, for Eq.~(\ref{eq:cc_2}) to hold (\textit{e.g.} $\ell_{rec}(f_\eta^*(\xx),k) \geq \ell_{rec}(f^*(\xx),k)$), it must be the case that $p_k=0,\;\forall k\neq y$, that is, $\ell_{rec}(f_\eta^*(\xx),k) = \ell_{rec}(f^*(\xx),k)$ for all $k \in \{1,2, \cdots, K\}$, thus $f^{\ast}_{\eta}=f^{\ast}$ which completes the proof.
\end{proof}

\section{Gradient Derivation of SL}\label{appendix_gradient}
The complete derivartion of the simplified SL ($\alpha, \beta = 1$) with respect to the logits is as follows:
\begin{equation}
\label{sce_dev_appd}
        \frac{\partial \ell_{sl}}{\partial z_j}  = - \sum_{k=1}^K q_k \frac{1}{p_k} \frac{\partial p_k}{\partial z_j}
        - \sum_{k=1}^K \frac{\partial p_k}{\partial z_j} \log q_k,
\end{equation}
where
\begin{equation}
        \frac{\partial p_k}{\partial z_j}  = \frac{\partial \Big(\frac{e^{z_k}}{\sum_{j=1}^K e^{z_j}}\Big)}{\partial z_j} = \frac{\frac{\partial e^{z_k}}{\partial z_j} (\sum_{j=1}^K e^{z_j}) - e^{z_k} \frac{\partial \big(\sum_{j=1}^K e^{z_j}\big) }{\partial z_j} }{(\sum_{j=1}^K e^{z_j})^2}.
\end{equation}
In the case of $ k = j$:
\begin{equation}
\label{ieqj}
    \begin{split}
        \frac{\partial p_k}{\partial z_j} & = \frac{\partial p_k}{\partial z_k} = \frac{e^{z_k}\big(\sum_{k=1}^K e^{z_k}\big) - (e^{z_k})^2}{(\sum_{k=1}^K e^{z_k})^2}\\
        & = \frac{e^{z_k}}{\sum_{k=1}^K e^{z_k}} - \Big(\frac{e^{z_k}}{\sum_{k=1}^K e^{z_k}} \Big)^2 \\
        & = p_k -p_k^2 = p_k(1-p_k);
    \end{split}
\end{equation}
In the case of $k \neq j$:
\begin{equation}
\label{ineqj}
    \begin{split}
        \frac{\partial p_k}{\partial z_j} & = \frac{0\cdot (\sum_{j=1}^K e^{z_j}) - e^{z_k}e^{z_j}}{(\sum_{j=1}^K e^{z_j})(\sum_{j=1}^K e^{z_j})} \\
        & = - \frac{e^{z_k}}{\sum_{j=1}^K e^{z_j}} \frac{e^{z_j}}{\sum_{j=1}^K e^{z_j}} \\
        & = -p_k p_j.
    \end{split}
\end{equation}
Combining Eq.~\eqref{ieqj} and \eqref{ineqj} into Eq.~\eqref{sce_dev_appd}, we can obtain:
\begin{equation}
    \begin{aligned}
        \frac{\partial \ell_{sl}}{\partial z_j} & = - \sum_{k=1}^K q_k \frac{1}{p_k} \frac{\partial p_k}{\partial z_j}
        - \sum_{k=1}^K \frac{\partial p_k}{\partial z_j} \log q_k \\
        & = - \sum_{k \neq j}^K \frac{q_k}{p_k}(-p_j p_k) - \frac{q_j}{p_j}(p_j(1-p_j)) - \sum_{k \neq j}^K (-p_j p_k) \log q_k  - p_j(1-p_j)\log q_j \\
        & = p_j - q_j + p_j(\sum_{k=1}^K p_k \log q_k - \log q_j).
    \end{aligned}
\end{equation}
If $q_{j} = q_y = 1$, then the gradient of SL is:
\begin{equation}
\label{y_1}
\begin{split}
    \frac{\partial \ell_{sl}}{\partial z_j} & = p_j - q_j + p_j(\sum_{k=1}^K p_k \log q_k - \log q_j) \\
    & = (p_j - 1) + p_j((1-p_j)A -0) \\
    & = \frac{\partial \ell_{ce}}{\partial z_j} - (A p_j^2 - Ap_j).
\end{split}
\end{equation}
Else if $q_j = 0$, then
\begin{equation}
\label{y_0}
\begin{aligned}
    \frac{\partial \ell_{sl}}{\partial z_j} & = p_j - q_j + p_j(\sum_{k=1}^K p_k \log q_k - \log q_j) \\
    & = p_j + p_j((1-p_y)A - A) \\
    & = \frac{\partial \ell_{ce}}{\partial z_j} - A p_jp_y.
\end{aligned}
\end{equation}

\end{document}